\documentclass{article}

\PassOptionsToPackage{numbers}{natbib}

\usepackage[main, final]{neurips_2026}

\usepackage[utf8]{inputenc} 
\usepackage[T1]{fontenc}    
\usepackage{hyperref}       
\usepackage{url}            
\usepackage{booktabs}       
\usepackage{amsfonts}       
\usepackage{nicefrac}       
\usepackage{microtype}      
\usepackage{xcolor}         
\usepackage{graphicx}
\usepackage{dsfont}
\usepackage{subcaption}

\usepackage{xspace}

\usepackage{amsmath}
\usepackage{amssymb}
\usepackage{mathtools}
\usepackage{amsthm}
\usepackage{tabularx}
\usepackage{multirow}
\usepackage{makecell}
\usepackage{pifont}
\usepackage{float}
\usepackage{placeins}
\usepackage{wrapfig}
\usepackage{algorithmicx}
\usepackage{algcompatible}

\usepackage[capitalize,noabbrev]{cleveref}

\usepackage[disable,textsize=tiny]{todonotes}

\usepackage{amsmath,amsfonts,bm}

\newcommand{\cmark}{\ding{51}}%
\newcommand{\xmark}{\ding{55}}%

\def\eqref#1{equation~\ref{#1}}

\def\1{\bm{1}}

\def\vs{{\bm{s}}}

\DeclareMathAlphabet{\mathsfit}{\encodingdefault}{\sfdefault}{m}{sl}
\SetMathAlphabet{\mathsfit}{bold}{\encodingdefault}{\sfdefault}{bx}{n}

\newcommand{\E}{\mathbb{E}}

\newcommand{\fdiv}{$f$-divergence\xspace}
\newcommand{\fdivs}{$f$-divergences\xspace}

\newcommand{\Lf}{\calL_f(\theta)}

\newcommand{\clip}[3]{\text{clip}\left(#1,\ #2,\ #3\right)}

\newcommand{\Df}{\mathcal{D}_f}

\newcommand{\calL}{\mathcal{L}}

\newcommand{\calS}{\mathcal{S}}
\newcommand{\calA}{\mathcal{A}}

\newcommand{\calB}{\mathcal{B}}
\newcommand{\calP}{\mathcal{P}}

\newcommand{\ind}[1]{\mathds{1}_{\left\{#1\right\}}}
\newcommand{\traj}{\vs_{1:T}}

\theoremstyle{plain}
\newtheorem{theorem}{Theorem}[section]

\theoremstyle{definition}
\newtheorem{definition}[theorem]{Definition}

\theoremstyle{remark}

\floatstyle{ruled}
\newfloat{algorithm}{tbp}{loa}
\floatname{algorithm}{Algorithm}

\title{Learning To Sample From Diffusion Models Via Inverse Reinforcement Learning}

\author{%
  Constant Bourdrez \\
  École Normale Supérieure, Université PSL,\\ CNRS, Inria, DI ENS, 75005 Paris \\
  \texttt{constant.bourdrez@ens.psl.eu} \\
  \And
  Alexandre Vérine \\
  École Normale Supérieure, Université PSL,\\ CNRS, Inria, DI ENS, 75005 Paris \\
  \And
  Olivier Cappé \\
  École Normale Supérieure, Université PSL,\\ CNRS, Inria, DI ENS, 75005 Paris \\
}

\begin{document}

\maketitle

\begin{abstract}

  Diffusion models generate samples through an iterative denoising process guided by a pretrained neural network. Once the denoiser is fixed, the sampling algorithm itself (noise schedules, guidance scales, stochasticity profiles) still requires careful tuning, a process typically carried out through costly empirical grid search.
  In this work, we introduce an inverse reinforcement learning framework for learning sampling strategies without retraining the denoiser. We formulate the diffusion sampling procedure as a discrete-time finite-horizon Markov Decision Process whose actions correspond to optional modifications of the sampling dynamics. To optimize action scheduling, we avoid defining an explicit reward function and instead directly match the target behavior expected from the sampler using policy gradient techniques.
  We provide experimental evidence that this approach matches fine-tuned samplers and comes at a modest cost compared to grid search: on ImageNet-64, a single training run replaces exhaustive search at up to $18\times$ lower cost, with only 16\% overhead at inference.

\end{abstract}

\section{Introduction}
\label{sec:intro}

Diffusion models and related generative processes based on differential
equations have become a standard tool for modeling high-dimensional data,
with applications ranging from image synthesis to scientific and biological
modeling. Sample generation proceeds by simulating a reverse-time stochastic
or deterministic process that gradually transforms a simple Gaussian distribution into a complex data distribution.
This evolution is governed by a vector field implemented by a large neural network,
commonly referred to as the denoiser.

While the training of diffusion models has been extensively studied,
the sampling procedure itself remains a critical bottleneck.
A growing body of work proposes heuristic modifications of the
sampling dynamics, including restart sampling \cite{xu_restart_2023}, guidance mechanisms
\cite{ho2022classifierfreediffusionguidance, dhariwal_diffusion_2021},
noise schedule tuning \cite{lee_convergence_2023}, and stochasticity injection
\cite{karras_elucidating_2022}. Although often effective, these approaches are
typically designed in an ad hoc manner and rely on manually tuned hyperparameters, a process that is rarely acknowledged but can consume as many resources as training
the denoiser itself. For instance, the EDM framework \citep{karras_elucidating_2022},
one of the strongest image generation baselines, required extensive grid search over its stochasticity profile: a search over 3 hyperparameters on ImageNet-64
alone would require up to 2869 EFLOPs, exceeding
the 1645 EFLOPs needed to train the denoiser.
More generally, evaluating a single hyperparameter configuration requires
generating tens of thousands of samples to compute evaluation metrics, and the cost grows exponentially with the number
of hyperparameters and model complexity. In this work, we propose to learn sampling-time hyperparameters
automatically by framing diffusion sampling as a sequential decision-making problem.

\begin{figure}[t]
    \centering
    \includegraphics[width=0.95\textwidth]{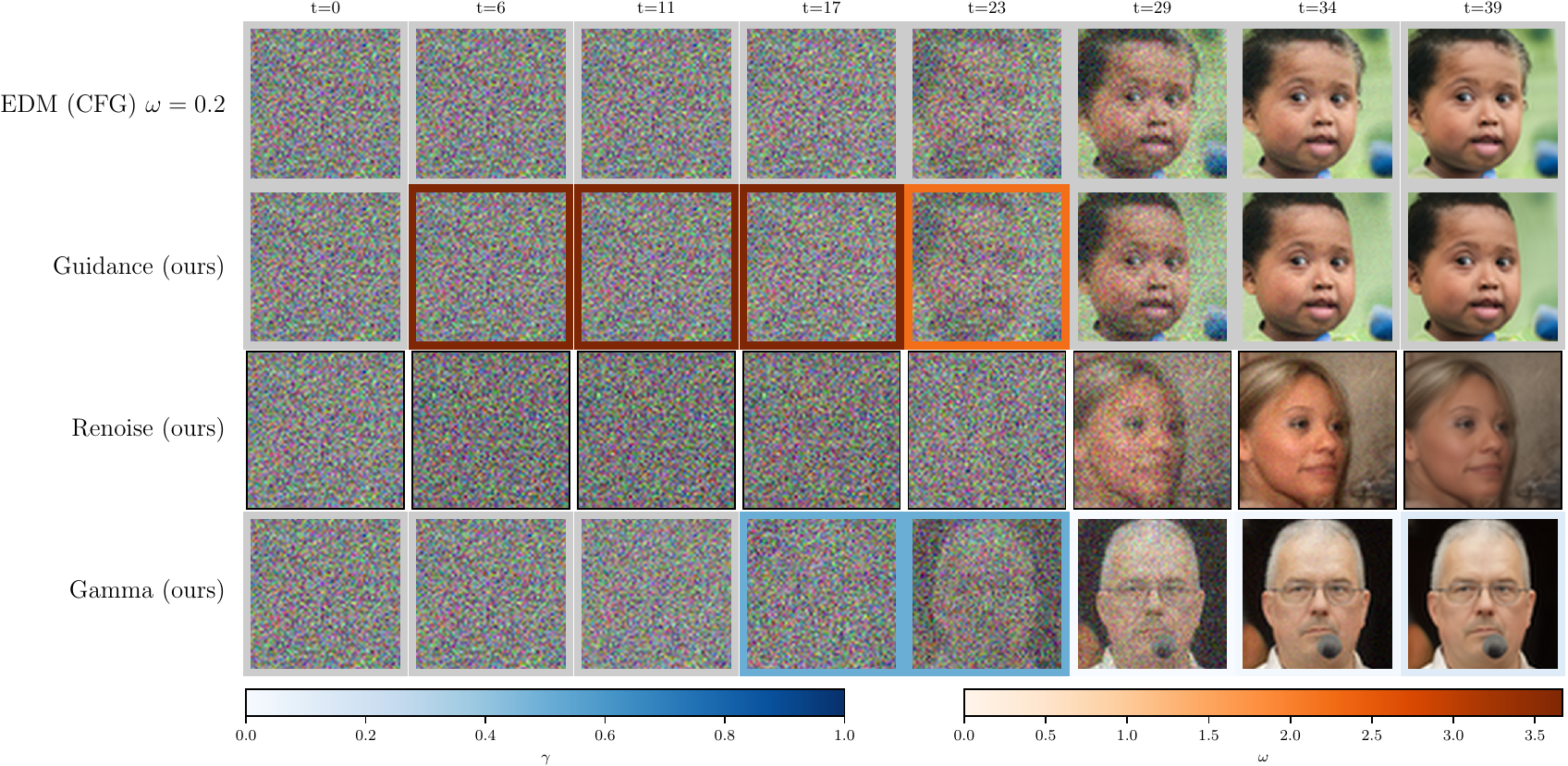}
    \caption{
        \textbf{Learned sampling policies adapt trajectories at inference time.}
        We compare trajectories from a standard guided diffusion sampler with trajectories produced by policies that learn guidance scales, renoising decisions, or stochasticity injection. All trajectories use the same pretrained denoiser; only the inference-time control policy changes. Frame colors indicate the selected action along the trajectory, showing that the learned sampler intervenes selectively rather than applying a fixed global schedule. These state-dependent controls adapt the sampler without denoiser retraining and reduce the need for manual hyperparameter tuning.
    }
    \label{fig:intro_trajectories}
\end{figure}

Because sampling is a trajectory optimization problem without an obvious per-step
reward, inverse reinforcement learning (IRL) provides a natural tool for learning
sampling policies from target state distributions
\citep{ziebart_maximum_nodate, ni_f-irl_2020, ho_generative_2016}.
However, existing IRL methods typically assume expert trajectories or actions,
whereas diffusion sampling provides only data marginals across noise levels.
RL and control approaches for diffusion models usually optimize explicit rewards or
fine-tune the denoiser \citep{black2023training, fan2023reinforcement, fan_optimizing_2023, guo_shortft_2025, uehara2024fine, azangulov2025adaptivediffusionguidancestochastic, domingoenrich2025adjointmatchingfinetuningflow, desanti2025provablemaximumentropymanifold}.
Learning the sampler itself through state-only IRL, while keeping the denoiser fixed,
has therefore remained largely unexplored.

\begin{table}[t]
    \centering
    \caption{Computational cost. \textit{Top}: training, \textit{Middle}: hyperparameter (HP) tuning  and our method \textit{Bottom}: inference for one sample. Wall-clock time and memory are reported in Appendix~\ref{app:wallclock}.}
    \label{tab:compute_cost}
    \small
    \vspace{2pt}
    \begin{tabular*}{\textwidth}{@{\extracolsep{\fill}}lcccc}
    \toprule
                                       & CIFAR-10 & FFHQ      & ImageNet-64 & ImageNet-512\footnotemark \\
    \midrule
    \multicolumn{5}{l}{\textit{Training Denoiser (EFLOPs)}}                                \\
    \midrule
    EDM \cite{karras_elucidating_2022} / EDM2 \cite{karras_analyzing_2024}
                                       & 25.7     & 50.9      & 1645.0      & 579.2        \\
    \midrule
    \multicolumn{5}{l}{\textit{Sampling Optimization (EFLOPs)}}                            \\
    \midrule
    Grid search, 2 HP, $5^2$           & 1.9      & 8.4       & 140.1       & 7.0          \\
    Grid search, 2 HP, $10^2$          & 7.5      & 33.5      & 560.4       & 28.1         \\
    Grid search, 3 HP, $5^3$           & 9.4      & 41.9      & 700.5       & 35.2         \\
    Grid search, 3 HP, $8^3$           & 38.3     & 171.6     & 2869.2      & 144.1        \\
    Ours                               & 10.2     & 40.3      & 160.6       & $9.7\times10^{-3}$ \\
    \midrule
    \multicolumn{5}{l}{\textit{Inference cost (TFLOPs, one sample)}}                       \\
    \midrule
    Denoiser size (resolution)         & 62M ($32^2$) & 66M ($64^2$) & 296M ($64^2$) & 125M ($512^2$) \\
    Denoiser only                      & 1.5      & 6.7      & 8.6      & 5.6      \\
    Denoiser $+$ policy                & 1.9      & 9.4      & 9.9      & 6.5      \\
    Overhead                           & $+$26\%  & $+$40\%   & $+$16\%     & $+$15\%      \\
    \bottomrule
\end{tabular*}

\end{table}
\footnotetext{ImageNet-512 uses the latent EDM2-XS model \cite{karras_analyzing_2024}, whose compact latent-space architecture makes it much cheaper than the pixel-space EDM models; its sampling costs include the VAE decode.}

We frame diffusion sampling as an IRL problem adapted to continuous states and
absorbing terminal dynamics. Rather than tuning hyperparameters by exhaustive search,
we learn a policy that matches an expert occupancy measure specified only through
data marginals at selected noise levels \textit{i.e} no expert trajectories, no reward model, no denoiser retraining.
Our method achieves results matching or exceeding the best hand-crafted sampling
strategies at a fraction of the cost: a single training run replaces grid search,
reducing the optimization budget by up to $18\times$ on ImageNet-64 as shown in Table~\ref{tab:compute_cost}, with only 16\% overhead at inference.
Our contributions are threefold: (i) we formulate diffusion sampling as an MDP for adapting samplers at inference time without retraining the denoiser, (ii) we derive a policy-gradient objective that directly matches state occupancy measures through $f$-divergences, and (iii) we validate the framework using guidance, stochasticity injection, and renoising policies on a representative set of state-of-the-art pretrained diffusion models.
Figure~\ref{fig:intro_trajectories} illustrates the type of state-dependent interventions learned by our policies.
Starting from the same pretrained denoiser and standard reverse process, the learned sampler can decide when to strengthen guidance, when to inject stochasticity, or when to revisit earlier noise levels.
This perspective is important: the denoiser is not modified, but the trajectory followed through the denoising process is.
The resulting policies therefore act as learned inference-time controls rather than as new generative models.

\section{Sample Generation with Diffusion Models}
\label{sec:diffusion}
Diffusion models generate data in $\mathbb{R}^d$ by transporting samples from a Gaussian distribution to a target data distribution \(P_0\), via a reverse-time stochastic or deterministic process \citep{song_bridging_2020,song_denoising_2022}. The reverse dynamics rely on a neural network \(F\) (the \emph{denoiser}) trained across noise levels to estimate the score of the noised data distribution. In practice, \(F\) is trained to predict the added noise \citep{ho_denoising_2020} or the denoised sample \citep{dhariwal_diffusion_2021}, and may be conditioned on auxiliary information \(c\) such as class labels or text prompts.

While diffusion models are commonly described in continuous time, practical
implementations rely on a finite discretization. We can see the generation process as a sequence of states \(s_t = (x_t, \sigma_t)\), where
\(x_t\) is the generated sample at time index \(t\) and \(\sigma_t\) is the corresponding noise level defined by the noise schedule
\(\{\Sigma_N, \ldots, \Sigma_0\}\) where $\Sigma_0 = 0 < \Sigma_1 < \cdots < \Sigma_N$. This sequence of noise levels defines a sequence of marginal distributions $\{ P_i \}_{i=0}^N$ obtained by
convolving the data distribution with Gaussian noise of variance $\Sigma_i^2$. We assume that $\Sigma_0$ is small enough for \(P_0\) to approximate the target data distribution, and that
$\Sigma_N$ is large enough for \(P_N\) to approximate pure Gaussian noise. The sequence starts at $t=0$ from \(s_0 = (x_0, \Sigma_N)\) with \(x_0 \sim \mathcal{N}(0, \Sigma_N^2 I_d)\) and ends at $t=T$ with \(s_T = (x_T, \Sigma_0)\), where \(x_T\) is expected to follow the target distribution \(P_0\). Each step of the reverse process maps \((x_t, \sigma_t)\) to \((x_{t+1}, \sigma_{t+1})\) according to
\begin{equation}
    (x_{t+1}, \sigma_{t+1}) = H(x_t, \sigma_t , \ \dots),
    \label{eq:generic_update}
\end{equation}
where \(H\) denotes a generic update operator that uses the denoiser \(F\), evaluated at the current state $(x_t,\sigma_t)$ and possibly additional inputs such as conditioning information \(c\). We treat $\sigma_t=\Sigma_0$ as an absorbing condition, i.e., once reached, the process remains at this noise level for all subsequent steps. The operator $H$ can represent a variety of samplers: it is stochastic for DDPM \citep{ho_denoising_2020} and Score-SDE \citep{song_score-based_2021}, and deterministic for probability-flow ODE (PF-ODE) based samplers, e.g., first-order Euler's method \citep{song_denoising_2022}, second-order Heun's method \citep{jolicoeur-martineau_gotta_2021,karras_elucidating_2022}, or higher-order solvers such as DPM-Solver \citep{lu_dpm-solver_2022}. Notably, $H$ depends on hyperparameters that influence the sampling dynamics. In this work, we focus on three popular techniques to modify the sampling process: stochasticity injection, classifier-free guidance, and restart sampling.

\textbf{Stochasticity Injection:} The EDM approach of \citet{karras_elucidating_2022}, which serves as a strong baseline for image generation, employs an hyperparameter \(\gamma_{\mathrm{EDM}}\) to introduce controlled stochasticity. Using $H_{\mathrm{Heun}}$ to denote the deterministic Heun update operator, the modified update rule relies on an amplified noise level $\hat{\sigma}_t\coloneq \sigma_t\,(1+\gamma_{\mathrm{EDM}})$ and a noise perturbation $z\sim \mathcal{N}(0,I)$:
\begin{align}
    H_{\mathrm{EDM}}(x_t, \sigma_t) & = H_{\mathrm{Heun}}(x_t + \sqrt{\hat{\sigma}_t^2 - \sigma_t^2}\, z, \hat{\sigma}_t).
\end{align}
For some datasets such as CIFAR-10 \citep{alex_krizhevsky_learning_2009} or FFHQ \citep{karras_style-based_2019}, $\gamma_{\mathrm{EDM}}=0$ yields optimal results, while for more complex datasets such
as ImageNet \citep{deng_imagenet_2009}, non-zero values can improve sample quality, requiring the joint tuning of four parameters that define the schedule of stochasticity injection.

\textbf{Classifier-Free Guidance:} Classifier-free guidance \citep{ho2022classifierfreediffusionguidance} is one of the most popular techniques to enhance sample quality in conditional diffusion models. For any $H$, where the model $F$ has been trained to be conditioned on $c$ or unconditioned ($c=\emptyset$), the update rule is kept identical to $H$ but the denoiser is modified as
\begin{align}
    F_{\mathrm{CFG}}(x_t, \sigma_t, c)
    = (1 + \omega) F(x_t, \sigma_t, c) - \omega F(x_t, \sigma_t, \emptyset).
\end{align}
Depending on the guidance scale \(\omega\), this technique can significantly improve sample fidelity, at the cost of diversity \citep{ho2022classifierfreediffusionguidance,saharia_photorealistic_2022}. However, there is no consensus on the optimal value of \(\omega\), nor on whether it should remain fixed or vary across the sampling trajectory.

\textbf{Restart Sampling:} Another approach consists in restarting the sampling process to escape poor local generations \citep{xu_restart_2023}. When the sample reaches \(\Sigma_{i}\), it is sent back \(K\) times to a higher noise level \(\Sigma_{j}\) ($j>i$) by adding Gaussian noise $z \sim \mathcal{N}(0,I)$, then resampled with a PF-ODE operator $H_{\mathrm{ODE}}$: denoting by $k_t$ the number of restarts up to time $t$,
\begin{align*}
    H(x_t, \sigma_t) = \begin{cases} H_{\mathrm{ODE}}(x_t + \sqrt{\Sigma_{j}^2 - \Sigma_{i}^2}\,z,\,\Sigma_{j}) & \text{if } k_t < K,\; \sigma_t = \Sigma_{i}, \\ H_{\mathrm{ODE}}(x_t, \sigma_t) & \text{otherwise.}\end{cases}
\end{align*}
The variant we consider (\emph{Renoise}) allows renoising to one of the $M{=}4$ previous noise levels upon reaching any restart level, with a fixed number of function evaluations (NFE) budget instead of a fixed restart count.

\section{State Distribution Matching with Inverse Reinforcement Learning}\label{sec:irl}
In this section, we review inverse reinforcement learning (IRL) tools that motivate our formulation in Section~\ref{sec:optimization}.
Our goal is to cast the diffusion sampling algorithm from Section~\ref{sec:diffusion} as a decision-making problem, in order to select sampler hyperparameters (e.g., guidance scales, injected noise, or renoise rules).
Crucially, we do \emph{not} assume access to an expert policy or expert actions.
Instead, we only observe samples from the data distribution and their noisy counterparts across noise levels, which naturally aligns with IRL settings based on matching state distributions.

\textbf{MDP formulation} We formalize diffusion sampling as a finite-horizon Markov Decision Process (MDP) $(\calS, \calA, \calP, \rho, T)$. We consider discrete-time policies with a fixed horizon $T$, initial states sampled from distribution $\rho$, and transition dynamics defined by the kernel $\calP$.
The state space $\calS$ is continuous, and the action space $\calA$ is discrete.
Let $\pi_\theta(a_t \mid s_t)$ be a parametrized policy over actions conditioned on states, and let $\traj = (s_1, \dots, s_T)$ denote a trajectory generated by this policy.
In Section~\ref{sec:optimization}, we will specify how each component of the MDP is instantiated for diffusion samplers.

\textbf{State occupancy measures}
In our setting, the expert policy is unknown and we do not observe expert actions or full trajectories; the only supervision comes from expert states (data samples and their noised versions).
This state-only regime motivates comparing policies through their \emph{state occupancy measures}, i.e., the marginal distribution over states visited along trajectories.

\begin{definition}[Occupancy Measure]
    The policy $\pi_\theta$ induces a state occupancy measure $\mu_\theta$ defined as
    \begin{equation}
        \forall \Phi \in \calB(\calS), \quad
        \int \Phi(s) \, \mu_\theta(ds)
        =
        \mathbb{E}_{\traj \sim p_\theta}
        \left[
            \frac{1}{T}\sum_{t=1}^T \Phi(s_t)
            \right],
    \end{equation}
    where $\calB(\calS)$ is the set of bounded measurable functions on $\calS$, and $p_\theta$ is the trajectory distribution induced by $\pi_\theta$.
    \label{def:occupancy_measure}
\end{definition}

When $\calS$ is discrete, $\mu_\theta$ admits the classical interpretation of an expected visitation frequency: $\mu_\theta(s)
    =
    \E_{\traj \sim p_\theta}
    \big[
        \frac{1}{T}
        \sum_{t=1}^T
        \ind{s_t = s}
        \big].$
We will not assume $\calS$ is discrete in what follows; the discrete case is only meant as intuition.
Similarly, we define the expert occupancy measure $\mu_E$ as the state occupancy measure induced by the (unknown) expert policy under the same environment dynamics.

\textbf{State-marginal matching via $f$-divergences}
State-marginal matching IRL seeks a policy whose state occupancy measure $\mu_\theta$ matches the expert occupancy measure $\mu_E$.
Concretely, one minimizes a discrepancy between these probability measures where several approaches focus on the Kullback--Leibler (KL) divergence, typically via maximum-entropy RL and reward learning \citep{ziebart_maximum_nodate,ni_f-irl_2020}.
More generally, one can use the broader class of $f$-divergences \citep{csiszar1967}, which provides a unified way to quantify distribution mismatch.

\begin{definition}[$f$-divergence]
    The $f$-divergence between two probability measures $Q$ and $P$, induced by a convex function $f$ satisfying $f(1)=0$, is
    \begin{equation}
        D_f(Q \parallel P) \coloneq \mathbb{E}_P\!\left[ f \left( \frac{dQ}{dP} \right) \right].
        \label{eq:f_divergence}
    \end{equation}
\end{definition}

This family includes KL, Reverse-KL (rKL), and total variation distance. The choice of $f$ matters: different divergences induce different trade-offs e.g., mode-covering (KL) vs. mode-seeking behavior (rKL), which can change the qualitative properties of the learned policy \citep{nowozin_f-gan_2016,verine_quality_2024}.

In this work, we consider objectives of the form $\min_{\pi_\theta} \; D_f\!\left(\mu_E \,\|\, \mu_\theta \right)$ subject to the MDP dynamics induced by the sampler.

\textbf{Prior work and our approach}
Classical IRL approaches first learn a reward under maximum-entropy RL \citep{ziebart_maximum_nodate, ni_f-irl_2020}.
GAIL \citep{ho_generative_2016} minimizes a state--action occupancy divergence via adversarial training, and \citet{ghasemipour_divergence_2019} extend this to general $f$-divergences via convex duality.
Both require expert actions, which are unavailable in our setting.
$f$-PG \citep{agarwal_f_2023} focuses on terminal state matching, which yields high-variance gradients over long trajectories. A related IRL approach is proposed by \citet{yoon2024maximumentropyinversereinforcement},
who learn an energy-based reward to guide sampling.
Our approach (i) directly minimizes a state-marginal $f$-divergence without reward learning, (ii) targets intermediate noise levels rather than only the terminal state, and (iii) uses PPO-style clipping to stabilize training.

\section{Improving Diffusion with Inverse Reinforcement Learning}
\label{sec:optimization}
In this section, we present our theoretical contributions and introduce a general framework for controlling generative sequential processes.
We formalize diffusion sampling as a finite-horizon MDP (Section~\ref{sec:mdp}), derive a policy gradient for directly minimizing $f$-divergences between state occupancy measures (Section~\ref{sec:policy_optimization}), and describe the practical optimization procedure.

\subsection{The Markov Decision Process Model for Diffusion Sampling}\label{sec:mdp}
To apply our inverse reinforcement learning algorithm, we frame each previously defined component as an element of diffusion sampling.
The state space $\calS$ is continuous, and each state is defined as
\[
    s_t = (x_t, \sigma_t),
\]
where $\sigma_t \in \{\Sigma_0, \ldots, \Sigma_N\}$ denotes the noise level at time step $t$ and $x_t \in \mathbb{R}^d$ is the current sample. The action space $\calA$ is discrete and encodes control decisions applied to the sampling dynamics. For instance, actions can correspond to selecting different discretized guidance scales, noise injection levels, or whether to renoise and how much at each timestep.
The initial distribution $\rho$ corresponds to the marginal distribution at the highest noise level, the Gaussian distribution $\mathcal{N}(0, \Sigma_N^2 I_d)$. Given a state-action pair $(s_t, a_t)$, the transition kernel $\calP$ is induced by the discretized diffusion dynamics together with the selected control action.
The policy does not explicitly depend on time but the available actions may depend on the current timestep $t$. We will treat $\sigma_t = \Sigma_0$ as an absorbing condition, so that once the process reaches this noise level, it remains there for all subsequent timesteps. In practice, this can be implemented by defining $H$ such that $s_{t+1} = s_t$ whenever $\sigma_t = \Sigma_0$. In this setup, the average time required to reach the terminal noise level may be incorporated in the performance metric of a given policy.

As the state space is continuous, we rely on Definition~\ref{def:occupancy_measure}. Exploiting the structure of the state space, we further decompose the occupancy measure
as $\mu_\theta(s)
    =
    w_\theta(\sigma)
    \times
    p_\theta(x \vert \sigma),$ where $w_\theta(\sigma)$ denotes the weight assigned by the policy $\pi_\theta$ to noise level $\sigma$, and $p_\theta(x \vert \sigma)$ is the marginal
distribution of $x$ at noise level $\sigma$ under the policy. For samplers such as noise injection or guidance, $w_\theta$ is uniform over noise levels, since the process always
progresses from $\Sigma_N$ to $\Sigma_0$ in $T$ steps. For samplers with renoising, $w_\theta$ reflects the distribution over noise levels induced by the policy, which may spend more time at certain noise levels depending on the restart strategy.

\textbf{Expert occupancy measure}
In our framework, the expert is specified through samples from the target distribution at selected noise levels, without access to expert trajectories, actions, or transition dynamics. The expert occupancy measure $\mu_E(s) = w_E(\sigma) \times p_E(x \vert \sigma)$ has weights $w_E(\sigma)$ that are a design choice: we set them uniformly over $\{\Sigma_N, \dots, \Sigma_1\}$ and concentrate the remaining weight on $\Sigma_0$, so that the absorbing terminal states provide a dense learning signal and $w_E(\Sigma_0)$ acts as a proxy for the mean time the policy should spend at the terminal noise level.

\textbf{On the leverage of $w_E$}
By the data processing inequality, $\Df(\mu_E \Vert \mu_\theta) \geq \Df(w_E \Vert w_\theta)$, so $w_E$ directly shapes the optimization. Additionally, for KL and rKL the objective decomposes into a schedule term $\Df(w_E \Vert w_\theta)$ and a weighted sum of per-noise-level divergences $\Df(p_E(\cdot|\sigma) \Vert p_\theta(\cdot|\sigma))$:
\begin{align}
    \Df(\mu_E \Vert \mu_\theta)  & = \Df(w_E \Vert w_\theta) + \textstyle\sum_{\sigma} w_E(\sigma)\, \Df\!\left(p_E(\cdot \mid \sigma) \Vert p_\theta(\cdot \mid \sigma)\right) . 
\end{align}
\subsection{Learning Policies by Minimizing $f$-Divergences}
\label{sec:policy_optimization}
We now describe how to optimize sampling policies within the MDP framework introduced above. We directly optimize the policy to match the expert occupancy measure by optimizing the following:
\begin{equation}
    \min_{\theta}\Lf = \min_{\theta} \; \Df\!\left( \mu_E \,\|\, \mu_\theta \right).
\end{equation}

\textbf{Policy gradient optimization}
We show in Appendix~\ref{proof:thm_grad} that the gradient of the $f$-divergence objective with respect to the policy parameters, denoted $\nabla_\theta \Lf$, can be expressed in the form of a policy gradient.
\begin{theorem}
    \label{thm:grad_inf}
    Let $\pi_\theta$ be a policy parameterized by $\theta$ and let $\Lf = \Df( \mu_E \Vert \mu_\theta)$ be the $f$-divergence between the expert occupancy measure $\mu_E$ and the occupancy measure $\mu_\theta$ induced by the policy $\pi_\theta$. Then, the gradient of $\Lf$ with respect to $\theta$ is given by
    \begin{equation}
        \nabla_\theta \Lf
        =
        \E_{\traj \sim p_\theta}
        \left[
            \frac{1}{T}
            \sum_{t=1}^T
            \nabla_\theta\log\pi_\theta(a_t\vert s_t)
            \sum_{t'\geq t}
            h_f\!\left(
            \frac{\mu_E(s_{t'})}{\mu_\theta(s_{t'})}
            \right)
            \right],
    \end{equation}
    where $h_f(x)=f(x)-xf'(x)$ for all $x$ and  $f$ is the generator function of the \fdiv (see Appendix~\ref{sec:appendix_proofs} for details).
\end{theorem}

This result enables the use of standard policy gradient methods to optimize the sampling policy using Monte Carlo estimates obtained from rollouts of the current policy.
While this result seems simple and similar to other policy gradient theorems, the main novelty lies in the direct optimization of the policy instead of a reward function.

\textbf{Density ratio estimation}
Computing the learning signal $h_f(\mu_E(s_t)/\mu_\theta(s_t))$ requires estimating the density ratio $\mu_E/\mu_\theta$ at each visited state.
We do so by training a discriminator $D_\phi : \calS \to [0,1]$ to distinguish between expert states (data samples and their noised versions) and policy states (samples generated by the current policy) \citep{sugiyama_density_2010, goodfellow_generative_2014}.
The discriminator is retrained jointly with the policy at each iteration; full details and the training procedure are given in Appendix~\ref{app:density_ratio}.

\textbf{Sample efficiency and practical considerations}
Generating trajectories is computationally expensive, as each rollout requires
simulating a full diffusion sampling trajectory. To improve sample efficiency, we reuse trajectories
over multiple policy updates and add importance sampling corrections to the policy gradient. While this modification is theoretically sound, we found it to be unstable since it leads to high-variance gradient estimates.
We attempted to mitigate this issue by subtracting a baseline, a standard technique for reducing variance in policy gradient methods.
However, in our setting, there is no clear choice of baseline; we use the global batch mean $\bar{A}$ of the cumulative learning signals as a variance-reducing baseline, following standard PPO practice.

To further reduce variance, we adopt a Proximal Policy Optimization (PPO)-style clipped objective~\citep{schulman_proximal_2017},
which constrains the policy updates at each iteration as well as reward normalization.
Under the same assumptions as Theorem~\ref{thm:grad_inf}, the PPO-style clipped
gradient for minimizing an $f$-divergence is given by
\begin{equation}
    \nabla_\theta\mathcal{L}^{\text{\tiny{PPO}}}_f(\theta)
    =
    \E_{p_{\theta_0}}
    \Biggl[
        \sum_{t=1}^T\nabla_\theta
        \max\Bigl(
        r_t(\theta)\,\hat{A}_t,\;
        \operatorname{clip}\bigl(r_t(\theta), 1-\epsilon, 1+\epsilon\bigr)\,\hat{A}_t
        \Bigr)
    \Biggr],
    \label{eq:ppo}
\end{equation}
with the probability ratio $r_t$, the cumulative learning signal $A_t$, and its version $\hat{A}_t$ recentered by the global batch mean $\bar{A}$:
\begin{equation*}
    r_t(\theta)=\frac{\pi_\theta(a_t \vert s_t)}{\pi_{\theta_0}(a_t \vert s_t)},\qquad
    A_t = \sum_{t'\geq t} h_f\!\left(\frac{\mu_E(s_{t'})}{\mu_\theta(s_{t'})}\right),\qquad
    \hat{A}_t = \frac{1}{T}\bigl(A_t - \bar{A}\bigr).
\end{equation*} Finally, following common practice, we approximate
$h_f\!\left(\mu_E(s_t)/\mu_\theta(s_t)\right)$ by
$h_f\!\left(\mu_E(s_t)/\mu_{\theta_0}(s_t)\right)$ to avoid recomputing the
density ratio at every policy update. We provide a detailed Algorithm~\ref{alg:importance_sampling} in Appendix~\ref{app:training_details}.

\section{Experiments}
We evaluate whether state-occupancy matching can replace manual sampler tuning once the denoiser is fixed.
The experiments are organized around three sampling-time interventions: adaptive classifier-free guidance, stochasticity injection, and renoising.
The main results in Table~\ref{tab:all_results} support three claims: learned policies can match or improve strong hand-tuned samplers, the best intervention is dataset-dependent, and the choice of $f$-divergence induces a stable precision--recall trade-off.
The ablations then isolate the design choices that make these policies effective.

\begin{table}[t]
    \centering
    \caption{Results across all strategies. Bold indicates the best value within each dataset--strategy block; standard deviations over repeated runs are reported in Appendix~\ref{app:results_std}. As expected, no method dominates on all metrics: the $f$-divergence sets the precision--recall trade-off.}
    \label{tab:all_results}
    \vspace{2pt}
    {\small{\renewcommand{\arraystretch}{1.0}\setlength{\tabcolsep}{2.6pt}\setlength{\aboverulesep}{0pt}\setlength{\belowrulesep}{0pt}\setlength{\extrarowheight}{1.2pt}
  \begin{tabular}{llrrrr@{\hspace{6pt}}|@{\hspace{6pt}}llrrrr}
    \toprule
    Strategy & Method & FID $\downarrow$ & Prec $\uparrow$ & Rec $\uparrow$ & NFE & Strategy & Method & FID $\downarrow$ & Prec $\uparrow$ & Rec $\uparrow$ & NFE \\
    \midrule
    \multicolumn{6}{l@{\hspace{6pt}}|@{\hspace{6pt}}}{\textbf{CIFAR-10 $32\times32$}} & \multicolumn{6}{l}{\textbf{FFHQ} $64\times64$} \\
    \cmidrule{1-12}
    \multirow{3}{*}{CFG}                     & EDM+CFG & 2.07             & 80.6            & 70.9            & 35 & \multirow{3}{*}{CFG}                     & EDM+CFG & 3.11             & 90.7            & 87.8            & 79 \\
    & KL     & $\mathbf{2.02}$ & $79.8$          & $\mathbf{71.9}$ & 35 & & KL     & $\mathbf{3.03}$ & $90.6$          & $\mathbf{88.6}$ & 79 \\
    & rKL    & $2.38$          & $\mathbf{81.2}$ & $70.6$          & 35 & & rKL    & $3.04$          & $\mathbf{90.8}$ & $88.3$          & 79 \\
    \cmidrule{1-12}
    \multirow{3}{*}{$\gamma_{\mathrm{EDM}}$} & EDM    & \textbf{1.98}    & 78.7            & 72.9            & 35 & \multirow{3}{*}{$\gamma_{\mathrm{EDM}}$} & EDM    & 2.56             & 90.1            & \textbf{89.4}   & 79 \\
    & KL     & $2.01$          & $78.8$          & $\mathbf{73.1}$ & 35 & & KL     & $2.53$          & $90.4$          & $88.2$          & 79 \\
    & rKL    & $2.04$          & $\mathbf{79.4}$ & $72.3$          & 35 & & rKL    & $\mathbf{2.49}$ & $\mathbf{90.9}$ & $87.9$          & 79 \\
    \cmidrule{1-12}
    \multirow{3}{*}{Renoise}                 & EDM    & 3.42             & 76.7            & \textbf{72.5}   & 17 & \multirow{3}{*}{Renoise}                 & EDM    & \textbf{2.60}    & 90.5            & \textbf{89.0}   & 59 \\
    & KL     & \textbf{3.18}    & 77.3            & \textbf{72.5}   & 17.7 & & KL     & 3.04             & \textbf{90.9}   & 88.1            & 87.4 \\
    & rKL    & 3.21             & \textbf{78.1}   & 71.6            & 30.5 & & rKL    & 3.07             & \textbf{90.9}   & 87.9            & 68.2 \\
    \midrule
    \multicolumn{6}{l@{\hspace{6pt}}|@{\hspace{6pt}}}{\textbf{ImageNet} $64\times64$} & \multicolumn{6}{l}{\textbf{ImageNet} $512\times512$} \\
    \cmidrule{1-12}
    \multirow{3}{*}{$\gamma_{\mathrm{EDM}}$} & EDM    & \textbf{2.12}    & 85.9            & 91.7            & 511 & \multirow{3}{*}{CFG}                     & EDM2+CFG & \textbf{3.80}  & 42.3            & \textbf{77.6}   & 35 \\
    & KL     & 2.14             & 85.7            & \textbf{92.0}   & 511 & & KL     & 3.85             & 42.3            & 76.6            & 35 \\
    & rKL    & 2.24             & \textbf{86.3}   & 91.1            & 511 & & rKL    & \textbf{3.80}    & 42.3            & \textbf{77.6}   & 35 \\
    \cmidrule{1-12}
    \multirow{3}{*}{Renoise}                 & EDM    & 3.01             & \textbf{85.8}   & 91.8            & 350 & \multirow{3}{*}{$\gamma_{\mathrm{EDM}}$} & EDM2   & \textbf{8.42}    & 41.0            & 84.3            & 35 \\
    & KL     & 2.96             & 85.4            & \textbf{92.3}   & 51.8 & & KL     & 8.50             & 41.0            & \textbf{84.4}   & 35 \\
    & rKL    & \textbf{2.92}    & 85.4            & 92.1            & 52.1 & & rKL    & 8.50             & \textbf{41.6}   & 84.1            & 35 \\
    \bottomrule
  \end{tabular}}
}
\end{table}

\subsection{Experimental Setup}

Experiments are conducted on CIFAR-10 ($32\times32$), FFHQ ($64\times64$), and ImageNet ($64\times64$) using pretrained EDM models \citep{karras_elucidating_2022}, and on ImageNet ($512\times512$) using the latent EDM2-XS model \citep{karras_analyzing_2024}, which runs in the latent space of a pretrained VAE.
The expert conditional distribution $p_E(x|\sigma)$ is the empirical dataset distribution at noise level $\Sigma_0 = 0$ with additional Gaussian noise at level $\sigma$.
All policy generations are evaluated using FID \citep{heusel_gans_2017} with 50k samples, Precision and Recall \citep{kynkaanniemi_improved_2019} using the improved topological method of \citet{kim_toppr_2023}, and the number of function evaluations (NFE) when relevant.

\textbf{Computational cost}
Table~\ref{tab:compute_cost} compares training and inference costs of our method against systematic hyperparameter tuning via grid search.
As the number of hyperparameters or the denoiser complexity grows, grid search, which follows the approach of \citet{karras_elucidating_2022}, becomes prohibitively expensive by requiring to pay the computational cost of FID evaluation for every hyperparameter configuration. In contrasts, with our approach, the overall cost comes from training the policy only once. At inference time, the policy adds between 15\% and 40\% overhead in FLOPs, less than one extra denoiser evaluation, and this overhead shrinks as the denoiser grows; quality and NFE are reported in Table~\ref{tab:all_results}.
\subsection{Results}
\label{sec:results}
We evaluate three strategies, all learned with the same state-occupancy matching objective and without modifying the denoiser. Table~\ref{tab:all_results} reports FID, Precision, Recall, and NFE.
Baselines are chosen per strategy: best constant-guidance CFG from a grid search, the EDM stochasticity hyperparameters from \citet{karras_elucidating_2022}, and EDM with reduced NFE for comparison to the adaptive renoising policies.

\begin{figure}[t]
    \centering
    \begin{subfigure}[t]{0.49\textwidth}
        \centering
        \includegraphics[width=\linewidth]{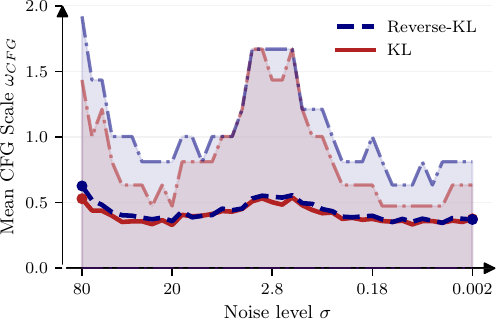}
        \caption{Adaptive guidance on FFHQ}
        \label{fig:cfg_profiles_ffhq}
    \end{subfigure}
    \begin{subfigure}[t]{0.49\textwidth}
        \centering
        \includegraphics[width=\linewidth]{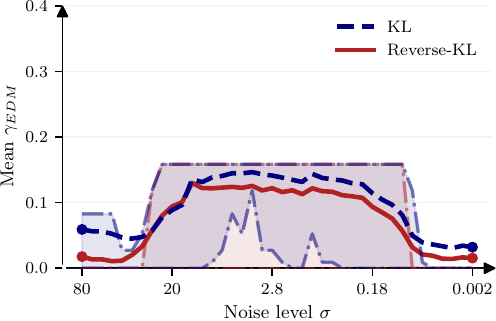}
        \caption{Adaptive stochasticity on FFHQ}
        \label{fig:gamma_profiles_ffhq}
    \end{subfigure}
    \hfill
    \caption{Learned $\gamma_{\mathrm{EDM}}(x,\sigma)$ and $\omega(x,\sigma)$ profiles on FFHQ for different \fdivs. At each noise level, the solid curve is the average value selected by the policy across generated samples, and the shaded band shows the interquartile range (25th--75th percentiles).}
    \label{fig:profiles_ffhq}
\end{figure}

\textbf{Analysis 1: adaptive guidance improves the quality--diversity trade-off}
We learn an adaptive guidance schedule $\omega(x,\sigma)$ on CIFAR-10 and FFHQ using conditional EDM models.
The action space consists of a discrete set of guidance scales; the baseline is the best constant guidance value found by grid search ($\omega{=}0.2$ for CIFAR-10, $\omega{=}0.1$ for FFHQ).
Adaptive policies improve the fixed-guidance baseline on the main quality--diversity trade-off.
KL achieves better Recall (diversity), while rKL improves Precision (fidelity).
The learned profile in Figure~\ref{fig:cfg_profiles_ffhq} reveals that the policy applies high guidance to a small fraction of samples at each timestep, concentrating this effect at medium-to-low noise levels corresponding to the speciation phase of sampling \citep{biroli_dynamical_2024}, where topological structures in the images emerge.

\textbf{Analysis 2: stochasticity injection is hard to hand tune}
We learn a state-dependent stochasticity profile $\gamma_{\mathrm{EDM}}(x,\sigma)$ on CIFAR-10, FFHQ, and ImageNet.
The action space consists of a discrete set of noise injection levels; the baseline is EDM with the stochasticity hyperparameters proposed by \citet{karras_elucidating_2022}.
The learned policy matches or improves this baseline on most metrics, while retaining the same denoiser and number of function evaluations.
On FFHQ (Figure~\ref{fig:gamma_profiles_ffhq}), the learned profile recovers the qualitative EDM pattern: little stochasticity at early noise levels and almost none near the end, with KL injecting more noise overall to increase diversity.
The CIFAR-10 profiles are more state-dependent (Appendix~\ref{app:gamma_profiles}, Figure~\ref{fig:gamma_profiles}), illustrating why exhaustive search over the four EDM hyperparameters is hard to replace manually.

\textbf{Analysis 3: renoising trades quality against computation}
We learn an adaptive renoising schedule on CIFAR-10, FFHQ, and ImageNet.
At each step, the policy chooses to continue denoising or revert to one of $M{=}4$ previous noise levels; the baseline is EDM run with a reduced NFE budget chosen to compare against the compute regime reached by adaptive renoising.
Renoising improves FID on CIFAR-10 and ImageNet, and substantially reduces NFE on ImageNet relative to the restart-style EDM baseline (Table~\ref{tab:all_results}); on FFHQ, however, the deterministic EDM sampler remains stronger in FID, leaving little room for adaptive renoising.
KL-based policies reach the terminal state faster (lower NFE), while rKL allocates more computation to high-noise regions, resulting in higher NFE without additional FID gains.

\textbf{Analysis 4: divergence choice controls precision and recall}
Across the three interventions, the choice of divergence provides a direct and interpretable mechanism for controlling the precision--recall trade-off \citep{nowozin_f-gan_2016, verine_quality_2024}.
KL is mode-covering: it encourages the policy to visit diverse regions of the data distribution, improving Recall at the cost of Precision.
rKL is mode-seeking: it concentrates probability mass on high-density regions, improving Precision at the cost of Recall.
This behavior is visible across guidance, stochasticity injection, and renoising, and is consistent with the occupancy-measure decomposition of Section~\ref{sec:policy_optimization}.

\subsection{Ablation Studies}
The ablations address the main implementation questions raised by the learned-policy formulation: whether stochastic policies are necessary, whether the method is sensitive to the selected divergence, whether state conditioning matters, and how much architecture/action-space design contributes to performance.

\textbf{Temperature}
The policy entropy can be adjusted at inference time via a temperature parameter $\beta$ by sampling action $a_t$ from $\pi_\theta(a_t \mid s_t)^{1/\beta} / Z_\beta$.
As shown in Figure~\ref{fig:temperature_pr_nfe}, higher $\beta$ increases NFE as the policy more frequently revisits higher-noise states, but Precision and Recall remain stable across a wide range, confirming robustness to entropy changes.
Lower $\beta$ (more deterministic) degrades FID, suggesting that deterministic schedules commonly used in practice may be inherently sub-optimal.
A similar behavior is observed for $\gamma_{\mathrm{EDM}}$ policies (Appendix~\ref{app:temp_gammas}).
\begin{figure}[t]
    \centering
    \includegraphics[width=\textwidth]{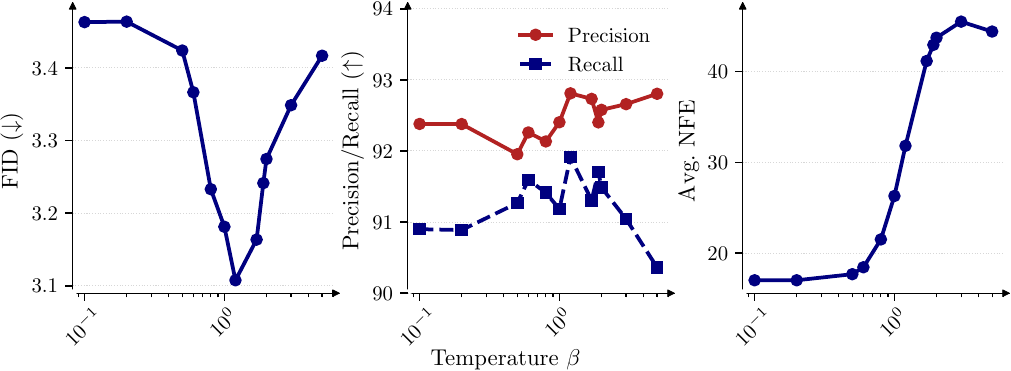}
    \caption{Effect of temperature $\beta$ on Precision, Recall, and NFE for learned renoise policies on CIFAR-10. FID increases for $\beta < 1$ highlighting the importance of stochasticity,
        while higher $\beta$ increases NFE without significant Precision/Recall changes. }
    \label{fig:temperature_pr_nfe}
\end{figure}

\textbf{Choice of $f$-divergence}
We evaluate additional $f$-divergences beyond KL and rKL for adaptive CFG on CIFAR-10 (Appendix~\ref{app:cfg_fdivs}, Table~\ref{tab:cfg_fdivs_cifar10}).
All choices yield competitive FID scores, confirming that the framework is robust to this hyperparameter.
However, symmetric divergences (TV, Hellinger, JS) do not exhibit the same strong mode-covering or mode-seeking tendency as KL or rKL, and produce less interpretable guidance profiles.
KL and rKL therefore remain the recommended choices when a controlled precision--recall trade-off is desired.

\begin{figure}[h!]
    \centering
    \begin{minipage}[t]{0.41\textwidth}
        \vspace{0pt}
        \centering
        \captionof{table}{State conditioning ablation for adaptive CFG on CIFAR-10: policies conditioned on the noise level $\sigma$, on the sampler state $(x,\sigma)$, or on $(x,\sigma,c)$ with the class label. Bold: best within each divergence block.}
        \label{tab:ablation_conditioning}
        \vspace{2pt}
        \resizebox{\linewidth}{!}{\small\setlength{\tabcolsep}{3pt}{\renewcommand{\arraystretch}{1.0}
\begin{tabular}{llrrr}
  \toprule
  Divergence & Policy            & FID $\downarrow$ & Prec $\uparrow$ & Rec $\uparrow$ \\
  \midrule
  \multirow{3}{*}{KL}
             & $\pi(\sigma)$     & $2.26$           & $80.3 $              & $71.1 $           \\
             & $\pi(x,\sigma)$   & $\mathbf{2.02}$  & $79.8$               & $\mathbf{71.9}$   \\
             & $\pi(x,\sigma,c)$ & $2.12 $          & $\mathbf{81.1} $     & $70.6 $           \\
  \midrule
  \multirow{3}{*}{rKL}
             & $\pi(\sigma)$     & $2.25 $          & $80.3 $              & $\mathbf{71.1} $  \\
             & $\pi(x,\sigma)$   & $2.38$           & $\mathbf{81.2}$      & $70.6$            \\
             & $\pi(x,\sigma,c)$ & $\mathbf{2.13} $ & $81.1 $              & $70.4 $           \\
  \midrule
  EDM + CFG  & ---               & 2.07             & 80.6                 & 70.9              \\
  \bottomrule
\end{tabular}}
}
    \end{minipage}
    \hfill
    \begin{minipage}[t]{0.555\textwidth}
        \vspace{0pt}
        \centering
        \includegraphics[width=1.\linewidth]{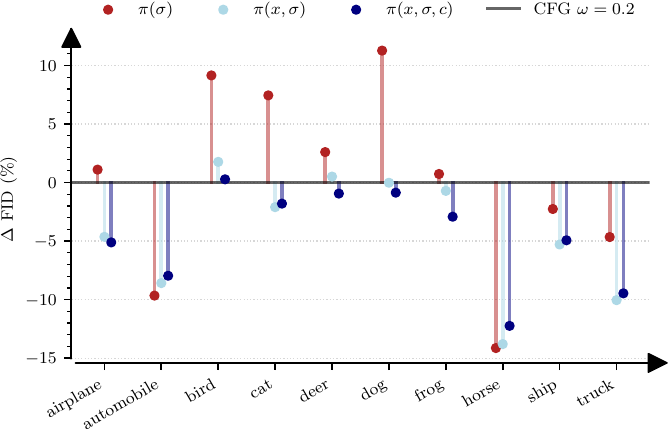}
        \caption{Per-class FID for CFG policies on CIFAR-10: conditioning on $x$ improves the policy.}
        \label{fig:class_fid_guidance}
    \end{minipage}
\end{figure}

\textbf{State conditioning}
We compare three conditioning variants for the adaptive CFG policy on CIFAR-10: $\pi(\sigma)$ (noise level only), $\pi(x, \sigma)$ (state-dependent), and $\pi(x, \sigma, c)$ (state-dependent with explicit class label $c$).
Table~\ref{tab:ablation_conditioning} reports results for both KL and rKL divergences.
Adding $x$ to the conditioning brings consistent FID improvements over $\pi(\sigma)$, confirming that per-sample adaptivity is essential and that noise-level alone is insufficient to determine the optimal guidance scale.
Crucially, $\pi(x, \sigma)$ achieves performance close to $\pi(x, \sigma, c)$, indicating that the policy implicitly recovers class-discriminative information from the image content $x$ alone without requiring explicit class supervision.
Figure~\ref{fig:class_fid_guidance} further shows that this gain is not driven by a single class: the state-dependent policy improves over the noise-level-only policy across most classes while remaining close to explicit class conditioning.
This suggests that the state representation carries enough structure for the policy to identify the modality of each sample and adapt guidance accordingly. Appendix~\ref{app:cfg_x_dependent} provides a qualitative trajectory comparison illustrating how state-dependent guidance changes the sampling path.

\begin{wraptable}{r}{0.47\textwidth}
    \vspace{-\intextsep}
    \centering
    \caption{Combining $\gamma_{\mathrm{EDM}}$ and CFG (NFE${}=35$). \emph{Composition}: independently trained policies applied jointly, without retraining; \emph{Multi-head}: one policy predicting both actions.}
    \label{tab:combined_policies}
    \label{tab:multihead}
    \resizebox{\linewidth}{!}{{\renewcommand{\arraystretch}{1.0}
  \begin{tabular}{llrrr}
    \toprule
    Strategy & Method & FID $\downarrow$ & Prec $\uparrow$ & Rec $\uparrow$ \\
    \midrule
    \multicolumn{5}{l}{\textbf{CIFAR-10} $32\times32$} \\
    \midrule
    \multirow{2}{*}{Baseline}    & EDM     & 1.98          & 78.7          & 72.9          \\
                                 & EDM+CFG & 2.07          & 80.6          & 70.9          \\
    \cmidrule{1-5}
    \multirow{2}{*}{Composition} & KL      & \textbf{1.75} & 79.0          & 73.2          \\
                                 & rKL     & 1.79          & 78.9          & \textbf{73.3} \\
    \cmidrule{1-5}
    \multirow{2}{*}{Multi-head}  & KL      & 2.03          & 80.2          & 70.5          \\
                                 & rKL     & 4.81          & \textbf{82.2} & 65.5          \\
    \midrule
    \multicolumn{5}{l}{\textbf{ImageNet} $512\times512$} \\
    \midrule
    \multirow{2}{*}{Baseline}    & EDM2     & 8.42          & 41.0          & \textbf{84.3} \\
                                 & EDM2+CFG & \textbf{3.80} & 42.3          & 77.6          \\
    \cmidrule{1-5}
    Composition                  & rKL      & 3.83          & \textbf{42.4} & 76.8          \\
    \bottomrule
  \end{tabular}}
}
    \vspace{-0.5\intextsep}
\end{wraptable}
\textbf{Action space and policy backbone}
Action space design and backbone choice both significantly impact performance (Appendix~\ref{app:action_space}, \ref{app:adm_backbone}, Tables~\ref{tab:ablation_n_gamma},~\ref{tab:adm_or_not}).
Each strategy admits an optimal action space cardinality beyond which performance saturates, confirming that expressivity matters but should be kept moderate.
Using a pretrained ADM encoder \cite{dhariwal_diffusion_2021} as a frozen feature extractor is crucial: removing it causes a large FID drop ($2.05 \to 3.04$), as rich perceptual features greatly facilitate policy and density-ratio training.

\textbf{Combining policies and training multi-head policies}
We combine guidance and stochasticity injection, either by applying independently trained policies jointly at inference, or by training a single two-head policy (Table~\ref{tab:combined_policies}).
Renoising is excluded, as it modifies the trajectory and conflicts with the other actions.
Without retraining, composition reduces FID on CIFAR-10 from $1.98$ to $1.75$, and matches the tuned EDM2+CFG baseline on ImageNet-512 ($3.83$ vs.\ $3.80$).
The joint KL policy improves over EDM+CFG ($2.03$ vs.\ $2.07$) but not over composition, while rKL collapses to high-Precision, low-Recall samples.

\section{Conclusion}
We introduced an inference-time framework for adapting diffusion sampling by optimizing sampling policies via inverse reinforcement learning, without retraining the denoiser or learning explicit rewards.
Experiments on CIFAR-10, FFHQ, and ImageNet indicate improved cost--quality trade-offs across multiple sampling strategies.
From a practical standpoint, the framework replaces costly hyperparameter grid searches with a single training run.
On ImageNet-64, this reduces the computational budget from up to 2869 EFLOPs (3-hyperparameter grid search, $8^3$ trials) to 161 EFLOPs, while adding only 16\% overhead at inference time.
This makes systematic sampling optimization feasible even for large-scale models where per-trial evaluation costs are prohibitive.

\paragraph{Discussion and limitations.}
Our experiments suggest that learned samplers are most useful when the policy is state-dependent and that its action is difficult to summarize by a small hand-tuned deterministic schedule. The improvement (over fixed settings) is clear for adaptive guidance and stochasticity injection, while results obtained for the renoising strategy are more contrasted as the benefit depends on the dataset. The current framework also depends on discrete action-space design and can be sensitive to policy initialization. Finally, the policy is trained on the same data as the denoiser; whether it transfers when the two differ remains an open question. Future work should therefore study more robust initialization, continuous or adaptive action spaces, and extensions to stronger pretrained models and other generative processes.

\newpage

\section*{AI use statement}
A generative AI assistant was used to support literature searches, discuss mathematical expressions
and experimental results, and assist with manuscript drafting, editing, and \LaTeX  preparation. The
authors take responsibility for all claims, citations, derivations, reported results, and the final content
of the paper.
\begin{ack}
  This work was granted access to the HPC resources of IDRIS under the allocations 2025-AD011016159R1, 2025-A0181016159 and 2025-AD011017138 made by GENCI. This project was partly funded by the French government's Agence Nationale de la Recherche under the "France 2030" program (ANR-23-IACL-0008, PRAIRIE-PSAI and ANR-22-PESN-0014).
\end{ack}

\bibliographystyle{plainnat}
\bibliography{references.bib}

\appendix
\renewcommand{\topfraction}{0.9}
\renewcommand{\bottomfraction}{0.8}
\renewcommand{\textfraction}{0.07}
\renewcommand{\floatpagefraction}{0.85}
\setcounter{totalnumber}{4}

\section*{Appendix}
\addcontentsline{toc}{section}{Appendix}

\renewcommand{\thefigure}{\thesection.\arabic{figure}}
\renewcommand{\thetable}{\thesection.\arabic{table}}
\setcounter{section}{0}
\renewcommand{\thesection}{\Alph{section}}

\section{Proofs}
\label{sec:appendix_proofs}

\subsection{Gradient formula proof}
\begin{proof}
    \label{proof:thm_grad}
   Let's consider our objective function:
   \[\Lf = \Df( \mu_E \Vert  \mu_\theta)\]

    We can compute the gradient of $\Lf$ with respect to $\theta$ using the chain rule and the properties of the \fdiv:
    \begin{align*}
    \nabla_\theta \Lf 
    &= \nabla_\theta \Df( \mu_E \Vert  \mu_\theta) \\
    &= \nabla_\theta \int_{\mathcal{S}} f\left(\frac{ \mu_E(s)}{ \mu_\theta(s)}\right)  \mu_\theta(s)\, ds \\
    &= \int_\mathcal{S} \nabla_\theta \mu_\theta(s)\, f\left(\frac{ \mu_E(s)}{ \mu_\theta(s)}\right)
        -  \mu_\theta(s) \frac{\nabla_\theta \mu_\theta(s)}{ \mu_\theta(s)^2}  \mu_E(s) \dot{f}\left(\frac{ \mu_E(s)}{ \mu_\theta(s)}\right)\, ds \\
    &= \int_\mathcal{S} \nabla_\theta \mu_\theta(s)\, h_f\left(\frac{ \mu_E(s)}{ \mu_\theta(s)}\right)\, ds \\
    &=\int_\mathcal{S} \nabla_\theta \log\mu_\theta(s) h_f\left(\frac{ \mu_E(s)}{ \mu_\theta(s)}\right)\,  \mu_\theta(s)ds
    \end{align*}

    with $h_f(x) = f(x) - x\dot{f}(x)$.

    Now, using the property of the state occupancy measure, we can write:
    \begin{align*}
        \nabla_\theta \Lf&=\E_{\traj \sim p_\theta}\left[\frac{1}{T}\sum_{t'=1}^{T}\nabla_\theta\log\mu_{\theta}(s_{t'})h_f\left(\frac{ \mu_E(s_{t'})}{ \mu_\theta(s_{t'})}\right)\right]\\
        &=\E_{\traj \sim p_\theta}\left[\frac{1}{T}\sum_{t'=1}^{T}\left(\sum_{t=1}^{t'}\nabla_\theta\log\pi_{\theta}(a_t\vert s_t)\right)h_f\left(\frac{ \mu_E(s_{t'})}{ \mu_\theta(s_{t'})}\right)\right] \\
        &=\E_{\traj \sim p_\theta}\left[\frac{1}{T}\sum_{t=1}^{T}\nabla_\theta\log\pi_{\theta}(a_t\vert s_t)\sum_{t'\geq t}h_f\left(\frac{ \mu_E(s_{t'})}{ \mu_\theta(s_{t'})}\right)\right]
    \end{align*}
    by reindexing the sums.
    \end{proof}

\subsection{Adding importance sampling}
\begin{proof}
    \label{proof:cor_is_gradient_inf}
    Adding importance sampling to the gradient formula derived in Theorem \ref{thm:grad_inf} gives:
    \begin{align*}
        \nabla_\theta \Lf &= \E_{\traj \sim p_{\theta_0}}\left[\frac{p_{\theta}(\traj)}{p_{\theta_0}(\traj)}\frac{1}{T}\sum_{t=1}^{T}\nabla_\theta\log\pi_{\theta}(a_t\vert s_t)\sum_{t'\geq t}h_f\left(\frac{ \mu_E(s_{t'})}{ \mu_\theta(s_{t'})}\right)\right]\\
    \end{align*}
    where $\theta_0$ is the parameter of the policy used to sample the trajectories.    
\end{proof}

\section{Experimental Details}
\label{app:experimental_details}

We parameterize the sampling policy and the state density estimator using a shared neural architecture, implemented as a dual-head network based on an EDM-style Encoder--UNet. The resulting model enables joint representation learning while keeping the policy and density objectives decoupled.

\textbf{Backbone and Preconditioning}
Given a noisy state $x$ and noise level $\sigma$, we apply the standard EDM preconditioning and map
$\sigma$ to the corresponding normalized sigma variable $\tau$. The network is
conditioned on $\tau$ and, when applicable, on class labels for conditional generation. For
non-stationary policies, the discrete trajectory step index is additionally provided as an explicit
conditioning variable. We also use a pretrained ADM encoder as
a frozen feature extractor. The input image is first processed by the ADM encoder, and
the resulting latent representation is used as input to the policy and density heads. All ADM
parameters are kept fixed throughout training.

\textbf{Policy Head}
The policy head outputs logits over a discrete action space of size $A$, where each action
corresponds to a sampling-time control operation (e.g., continuing denoising or restarting to a
previous noise level). The policy head is implemented as convolutional layers with output channels equal to
$A$, sharing the same architectural design and conditioning mechanisms as the density head. For
stationary policies, the policy is conditioned only on the current state and noise level.
For non-stationary policies, the trajectory step index is included as an additional conditioning
signal. During training and inference, actions are sampled from the resulting categorical
distribution.

\textbf{Density Head}
The density head estimates a scalar quantity associated with the state distribution, used for
density-ratio-based training objectives. It is implemented as convolutional layers with a single scalar
output. The density head shares the same backbone structure and conditioning as the policy head but
has independent parameters in the output layers. The output layers of the density head are
initialized using Xavier initialization to stabilize early training.

\FloatBarrier
\subsection{Training Details}
\label{app:training_details}

\textbf{Initialization}
The initialization of both the policy and density networks plays an important role in practice. Since the learning problem is
non-convex, we do not guarantee convergence to an unique policy, and different initializations may lead to different local
optima. We initialize the policy network using a task-specific heuristic corresponding to a reasonable sampling strategy
(for instance favoring continuation of denoising at high noise levels and allowing restarts at later stages). This initialization
provides a stable starting point and significantly improves training stability compared to random initialization.

The density network is initialized using expert trajectories and rollouts generated by the initialized policy. These expert
samples are used to initialize the density-ratio estimator before joint training with the policy, which helps stabilize the
early stages of optimization.

\textbf{Optimization}
For both the policy and density networks, we apply an exponential moving average (EMA) to the network parameters during
training \citep{polyak_acceleration_1992, karras_analyzing_2024}. EMA is used for evaluation and inference, as it consistently yields more stable behavior and improved performance. The density-ratio estimate is further stabilized by a 120-step warm-up of the density network before any policy update, and by predicting log-ratios rather than probabilities, which avoids sigmoid saturation.
Finally, when training the density estimator, we apply label smoothing at high noise levels. At large noise scales, the
input states contain limited information, and hard labels can lead to unstable gradients. Label smoothing in this regime
improves robustness and reduces variance during training.

We provide details of the complete training procedure in Table~\ref{tab:training_config}.

\begin{table}[H]
    \centering
    \caption{Parameters used for training the sampling policies across different datasets and strategies. Compute is reported in normalized GPU hours (V100-equivalent: one H100 hour counts as 4 normalized hours).}
    \label{tab:training_config}
    {\small\resizebox{\textwidth}{!}{%
\begin{tabular}{lcccccccccc}
\toprule
Parameters
 & \multicolumn{3}{c}{Renoise}
 & \multicolumn{4}{c}{Gammas}
 & \multicolumn{3}{c}{Guidance} \\
 & \rotatebox{90}{CIFAR-10}
 & \rotatebox{90}{FFHQ 64$\times$64}
 & \rotatebox{90}{ImageNet 64$\times$64}
 & \rotatebox{90}{CIFAR-10}
 & \rotatebox{90}{FFHQ 64$\times$64}
 & \rotatebox{90}{ImageNet 64$\times$64}
 & \rotatebox{90}{ImageNet 512$\times$512$^{\dagger}$}
 & \rotatebox{90}{CIFAR-10}
 & \rotatebox{90}{FFHQ 64$\times$64}
 & \rotatebox{90}{ImageNet 512$\times$512$^{\dagger}$} \\
\midrule
Cond. model     & $\times$ & $\times$ & \checkmark & $\times$ & $\times$ & \checkmark & \checkmark & \checkmark & \checkmark & \checkmark \\
LR              & 1e-4 & 1e-5 & 1e-5 & 1e-4 & 1e-5 & 1e-5 & 1e-5 & 1e-4 & 1e-5 & 1e-5 \\
\# traj.        & 1024 & 1536 & 2048 & 1024 & 1536 & 2048 & 512 & 1024 & 1536 & 256 \\
\# DRE Init.    & 100 & 200 & 200 & 100 & 200 & 200 & 400 & 100 & 200 & 400 \\
Action dim.     & 4 & 4 & 4 & 10 & 10 & 20 & 10 & 20 & 20 & 20 \\
Res.            & $32^2$ & $64^2$ & $64^2$ & $32^2$ & $64^2$ & $64^2$ & $512^2$ & $32^2$ & $64^2$ & $512^2$ \\
\# params.      & 5.8M & 5.8M & 6.8M & 5.8M & 5.8M & 6.8M & 0.37M$^{\ddagger}$ & 5.8M & 5.8M & 0.37M$^{\ddagger}$ \\
\# Noise levels & 9 & 20 & 20 & 18 & 40 & 256 & 18 & 18 & 40 & 18 \\
Norm.\ GPU hrs  & 30 & 48 & 320 & 60 & 80 & 640 & 348& 100 & 150 & 200 \\
GPU             & V100 & V100 & H100 & V100 & V100 & H100 & MI250 & V100 & V100 & A100 \\
\bottomrule
\end{tabular}}

\vspace{0.1cm}
\begin{minipage}{\linewidth}
\footnotesize
$^{\dagger}$~Latent diffusion on a frozen EDM2-XS backbone (125M params); policy/density
heads are trained on top of the frozen, shared EDM2 encoder rather than an independently
trained encoder.
$^{\ddagger}$~Trainable parameters of the policy/density heads only.
\end{minipage}
}
\end{table}

\FloatBarrier
\subsection{Density Ratio Estimation}
\label{app:density_ratio}

\textbf{Discriminator-based density ratio estimation}
Estimating density ratios in continuous spaces is a classical problem, with
approaches ranging from kernel density estimation to direct ratio fitting and
probabilistic modeling~\cite{sugiyama_density_2010}. In this work, we adopt a
classifier-based approach, which has been shown to scale effectively to
high-dimensional settings~\citep{ho_generative_2016, finn_connection_2016,song_bridging_2020, roth_stabilizing_2017,azadi_discriminator_2019} and has successfully been applied to diffusion models~\citep{kim_refining_2023,verine_improving_2025}.

We train a binary classifier $D_\varphi(x,\sigma)$ to discriminate between states sampled from
the expert and from the policy, conditioning on the noise level $\sigma$. Since
the expert provides only marginal distributions at each noise level, we sample
$\sigma$ uniformly and draw states independently from $p_E(\cdot\mid\sigma)$ and
$p_\theta(\cdot\mid\sigma)$. The classifier is trained by minimizing the binary
cross-entropy loss
\begin{equation*}
    \mathcal{L}_D(\varphi) = \E_{\sigma \sim \mathbb{U}[\Sigma_0,\ldots,\Sigma_N]} \Big[ - \E_{x \sim p_E(\cdot \mid \sigma)}\big[ \log D_\varphi(x,\sigma) \big] - \E_{x \sim p_\theta(\cdot \mid \sigma)}\big[ \log (1 - D_\varphi(x,\sigma)) \big] \Big].
\end{equation*}
Assuming an optimal discriminator, the density ratio is
\begin{equation}
    \frac{p_E(x \mid \sigma)}{p_\theta(x \mid \sigma)}
    =
    \frac{D_\varphi(x,\sigma)}{1 - D_\varphi(x,\sigma)}
    =
    \exp\!\left(\tilde{D}_\varphi(x,\sigma)\right),
\end{equation}
where $\tilde{D}_\varphi(x,\sigma)$ denotes the logit output of the classifier.

\begin{algorithm}[H]
    \caption{Policy Optimization via State-Marginal f-Divergence Minimization}
    \label{alg:importance_sampling}
    \begin{algorithmic}[1]
        \REQUIRE Pretrained denoiser $F$, policy $\pi_\theta$, discriminator $D_\varphi$, replay buffer $B$, number of epochs $n_{\text{epoch}}$, batch size $N_{\text{batch}}$, expert state samples $\mathcal{S}_E$, clipping parameter $\epsilon$, learning rate $\eta$, and policy update frequency $K$.
        \STATE Initialize policy $\pi_\theta$ and discriminator $D_\varphi$.
        \FOR{$k = 1$ \textbf{to} $n_{\text{epoch}}$}
        \IF{$k \bmod K = 0$}
        \STATE Generate trajectories
        \[
            \traj = \{(s_t, a_t)\}_{t=1}^T
        \]
        using policy $\pi_\theta$ and denoiser $F$.

        \STATE Train discriminator $D_\varphi$ to distinguish expert states $\mathcal{S}_E$ from policy-generated states $\{s_t\}$.

        \STATE Estimate density ratios using
        \[
            \frac{\mu_E(x,\sigma)}{\mu_\theta(x,\sigma)} \;\approx\;
            \frac{D_\varphi(x,\sigma)}{1 - D_\varphi(x,\sigma)}\times \frac{w_E(\sigma)}{w_\theta(\sigma)}.
        \]

        \STATE Compute advantage estimates
        \[
            A_t \;=\; \frac{1}{T}\sum_{t' = t}^{T}
            h_f\!\left(\frac{\mu_E(s_{t'})}{\mu_\theta(s_{t'})}\right),
            \quad
            h_f(x) = f(x) - x f'(x),
        \]
        and store $(s_t, a_t, A_t)$ in buffer $B$.

        \STATE Save current policy parameters: $\theta_0 \gets \theta$.
        \ENDIF

        \FOR{each minibatch $\{(s_i,a_i,A_i)\}_{i=1}^{N_{\text{batch}}}$ sampled from $B$}
        \STATE Compute PPO-style objective:
        \[
            \hat{\mathcal{L}}_f(\theta)
            =
            \frac{1}{N_{\text{batch}}}
            \sum_{i=1}^{N_{\text{batch}}}
            \max\!\left(
            r_i(\theta)\,(A_i - \bar{A}),
            \,
            \clip{r_i(\theta)}{1-\epsilon}{1+\epsilon}\,(A_i - \bar{A})
            \right),
        \]
        where
        \[
            r_i(\theta) = \frac{\pi_\theta(a_i \mid s_i)}{\pi_{\theta_0}(a_i \mid s_i)},
            \qquad
            \bar{A} = \frac{1}{N_{\text{batch}}}\sum_{i=1}^{N_{\text{batch}}} A_i .
        \]

        \STATE Update policy parameters:
        \[
            \theta \gets \theta - \eta \nabla_\theta \hat{\mathcal{L}}_f(\theta).
        \]
        \ENDFOR
        \ENDFOR
        \ENSURE Trained sampling policy $\pi_\theta$.
    \end{algorithmic}
\end{algorithm}

\textbf{Computing the occupancy measure ratio}
Because expert trajectories are unavailable, the expert is specified only through
state marginals at each noise level. As a result, the discriminator estimates a
density ratio under a uniform distribution over noise levels rather than under the
policy-induced state occupancies. In addition, while the policy accumulates
occupancy mass at the terminal noise level, the expert does not provide temporal
occupancy information.

We compute the full occupancy measure ratio by using importance sampling corrections to
account for the mismatch between sampling and target proportions.
\begin{equation}
    \frac{\mu_E(x, \sigma)}{\mu_\theta(x, \sigma)}
    =
    \frac{w_E(\sigma)}{w_\theta(\sigma)}
    \times
    \frac{p_E(x \vert \sigma)}{p_\theta(x \vert \sigma)}.
\end{equation}
As previously discussed, the term $w_E$ is a design choice while $w_\theta$ is induced by the policy and is estimated empirically from the simulated trajectories. In practice $w_E$ matters little when the number of noise levels matches the number of sampling steps, since the weighting is then nearly uniform; we observed an effect only in low-dimensional renoising experiments, where a larger terminal weight shortens trajectories.

\FloatBarrier
\subsection{Wall-Clock Time and Memory}
\label{app:wallclock}
Table~\ref{tab:compute_cost} reports the inference overhead of the policy in FLOPs. Table~\ref{tab:wallclock} complements it with measured wall-clock time and peak GPU memory for a single network evaluation (batch size $1$, V100 32GB, averaged over $50$ runs after $10$ warm-up runs).
The policy adds roughly $21$--$25$\,ms per evaluation, an approximately constant cost that becomes relatively smaller as the denoiser grows ($+103\%$ on CIFAR-10 vs.\ $+57\%$ on ImageNet-64), and a negligible memory footprint (at most $6$\,MB).

\begin{table}[H]
    \centering
    \caption{Wall-clock time and peak GPU memory of a single forward pass, with and without the policy (batch size $1$, V100 32GB, mean over $50$ runs after $10$ warm-up runs).}
    \label{tab:wallclock}
    \vspace{0.2cm}
    {\small\renewcommand{\arraystretch}{1.0}
  \begin{tabular}{lrrrrrr}
    \toprule
                         & \multicolumn{3}{c}{Wall-clock time (ms)} & \multicolumn{3}{c}{Peak memory (MB)}                     \\
    Dataset              & Denoiser & + Policy & Overhead              & Denoiser & + Policy & Overhead \\
    \midrule
    CIFAR-10 $32\times32$ & 20.80    & 42.24    & $+$21.44 ($+$103\%)   & 474      & 480      & $+$6     \\
    FFHQ $64\times64$     & 26.82    & 51.52    & $+$24.70 ($+$92\%)    & 589      & 589      & $+$0     \\
    ImageNet $64\times64$ & 43.89    & 68.85    & $+$24.96 ($+$57\%)    & 1565     & 1565     & $+$0     \\
    \bottomrule
  \end{tabular}}

\end{table}

\FloatBarrier
\subsection{Results with Variability}
\label{app:results_std}

The main paper reports point estimates in Table~\ref{tab:all_results} to keep the comparison readable. For guidance and stochasticity injection on CIFAR-10 and FFHQ, learned policies were trained and evaluated over five random seeds; Tables~\ref{tab:all_results_std} and~\ref{tab:ablation_conditioning_std} provide mean$\pm$standard deviation values when available; remaining entries are single point estimates.

\begin{table}[H]
    \centering
    \caption{Results across all strategies with mean$\pm$standard deviation when repeated runs are available. Variability is reported for repeated CIFAR-10 and FFHQ learned-policy evaluations; ImageNet results are single-run evaluations due to the cost of 50k-sample metrics at that scale.}
    \label{tab:all_results_std}
    {\small{\renewcommand{\arraystretch}{1.0}
{\renewcommand{\arraystretch}{1.0}
  \begin{tabular}{llrrrr}
    \toprule
    Strategy                                 & Method & FID $\downarrow$       & Prec $\uparrow$       & Rec $\uparrow$        & NFE  \\
    \midrule
    \multicolumn{6}{l}{\textbf{CIFAR-10 $32\times32$}}                                                                                \\
    \cmidrule{1-6}
    \multirow{3}{*}{CFG}                     & EDM+CFG & 2.07                  & 80.6                  & 70.9                  & 35   \\
                                             & KL     & $\mathbf{2.02}\pm0.15$ & $79.8\pm0.9$          & $\mathbf{71.9}\pm1.1$ & 35   \\
                                             & rKL    & $2.38\pm0.14$          & $\mathbf{81.2}\pm0.9$ & $70.6\pm0.6$          & 35   \\
    \cmidrule{1-6}
    \multirow{3}{*}{$\gamma_{\mathrm{EDM}}$} & EDM    & \textbf{1.98}          & 78.7                  & 72.9                  & 35   \\
                                             & KL     & $2.01\pm0.03$          & $78.8\pm0.2$          & $\mathbf{73.1}\pm0.3$ & 35   \\
                                             & rKL    & $2.04\pm0.02$          & $\mathbf{79.4}\pm0.3$ & $72.3\pm0.4$          & 35   \\
    \cmidrule{1-6}
    \multirow{3}{*}{Renoise}                 & EDM    & 3.42                   & 76.7                  & \textbf{72.5}         & 17   \\
                                             & KL     & \textbf{3.18}          & 77.3                  & \textbf{72.5}         & 17.7 \\
                                             & rKL    & 3.21                   & \textbf{78.1}         & 71.6                  & 30.5 \\
    \midrule

    \multicolumn{6}{l}{\textbf{FFHQ} $64\times64$}                                                                                    \\
    \cmidrule{1-6}
    \multirow{3}{*}{CFG}                     & EDM+CFG & 3.11                  & 90.7                  & 87.8                  & 79   \\
                                             & KL     & $\mathbf{3.03}\pm0.02$ & $90.6\pm0.2$          & $\mathbf{88.6}\pm0.1$ & 79   \\
                                             & rKL    & $3.04\pm0.01$          & $\mathbf{90.8}\pm0.1$ & $88.3\pm0.1$          & 79   \\
    \cmidrule{1-6}
    \multirow{3}{*}{$\gamma_{\mathrm{EDM}}$} & EDM    & 2.56                   & 90.1                  & \textbf{89.4}         & 79   \\
                                             & KL     & $2.53\pm0.01$          & $90.4\pm0.1$          & $88.2\pm0.2$          & 79   \\
                                             & rKL    & $\mathbf{2.49}\pm0.02$ & $\mathbf{90.9}\pm0.1$ & $87.9\pm0.1$          & 79   \\
    \cmidrule{1-6}
    \multirow{3}{*}{Renoise}                 & EDM    & \textbf{2.60}          & 90.5                  & \textbf{89.0}         & 59   \\
                                             & KL     & 3.04                   & \textbf{90.9}         & 88.1                  & 87.4 \\
                                             & rKL    & 3.07                   & \textbf{90.9}         & 87.9                  & 68.2 \\
    \midrule

    \multicolumn{6}{l}{\textbf{ImageNet} $64\times64$}                                                                                \\
    \cmidrule{1-6}
    \multirow{3}{*}{$\gamma_{\mathrm{EDM}}$} & EDM    & \textbf{2.12}          & 85.9                  & 91.7                  & 511  \\
                                             & KL     & 2.14                   & 85.7                  & \textbf{92.0}         & 511  \\
                                             & rKL    & 2.24                   & \textbf{86.3}         & 91.1                  & 511  \\
    \cmidrule{1-6}
    \multirow{3}{*}{Renoise}                 & EDM    & 3.01                   & \textbf{85.8}         & 91.8                  & 350  \\
                                             & KL     & 2.96                   & 85.4                  & \textbf{92.3}         & 51.8 \\
                                             & rKL    & \textbf{2.92}          & 85.4                  & 92.1                  & 52.1 \\
    \bottomrule
  \end{tabular}}}
}
\end{table}

The repeated runs support the qualitative conclusions of the main table. On CIFAR-10, the standard deviations of FID are small relative to the gaps between the main trends: adaptive CFG with KL remains competitive with fixed guidance while improving Recall, and rKL consistently shifts the trade-off toward Precision. For stochasticity injection, the gap to deterministic EDM is within a few hundredths of FID, but KL and rKL still produce the expected Recall/Precision shifts. For FFHQ, the reported variability is very small for CFG and $\gamma_{\mathrm{EDM}}$, suggesting that the observed gains are not driven by evaluation noise.

\begin{table}[H]
    \centering
    \caption{State conditioning ablation for adaptive CFG on CIFAR-10 with mean$\pm$standard deviation over repeated runs.}
    \label{tab:ablation_conditioning_std}
    {\small{\renewcommand{\arraystretch}{1.0}
\begin{tabular}{llrrr}
\toprule
Divergence & Policy & FID $\downarrow$ & Prec $\uparrow$ & Rec $\uparrow$ \\
\midrule
\multirow{3}{*}{KL}
  & $\pi(\sigma)$       & $2.26 \pm 0.01$ & $80.3 \pm 0.1$ & $71.1 \pm 0.2$ \\
  & $\pi(x,\sigma)$     & $\mathbf{2.02}\pm0.15$ & $79.8\pm0.9$ & $\mathbf{71.9}\pm1.1$ \\
  & $\pi(x,\sigma,c)$   & $2.12 \pm 0.05$ & $\mathbf{81.1} \pm 0.1$ & $70.6 \pm 0.1$\\
\midrule
\multirow{3}{*}{rKL}
  & $\pi(\sigma)$       & $2.25 \pm 0.01$ & $80.3 \pm 0.1$ & $\mathbf{71.1} \pm 0.2$          \\
  & $\pi(x,\sigma)$     & $2.38\pm0.14$ & $\mathbf{81.2}\pm0.9$ & $70.6\pm0.6$ \\
  & $\pi(x,\sigma,c)$   & $\mathbf{2.13} \pm 0.01$ & $81.1 \pm 0.3$ & $70.4 \pm 0.2$ \\
\midrule
EDM + CFG  & ---             & 2.07          & 80.6          & 70.9          \\
\bottomrule
\end{tabular}}
}
\end{table}

Table~\ref{tab:ablation_conditioning_std} confirms that the benefit of state conditioning is robust. For KL, conditioning on $(x,\sigma)$ improves FID over the $\sigma$-only policy and increases Recall, while the standard deviation remains modest. For rKL, explicit class conditioning gives the best FID, but $\pi(x,\sigma)$ still achieves the highest Precision and remains close to $\pi(x,\sigma,c)$, supporting the claim that the noisy image state already carries substantial class-discriminative information.

\FloatBarrier
\subsection{Learned CFG Guidance Profiles}
\label{app:cfg_profiles}

\begin{figure}[H]
    \centering
    \includegraphics[width=0.6\textwidth]{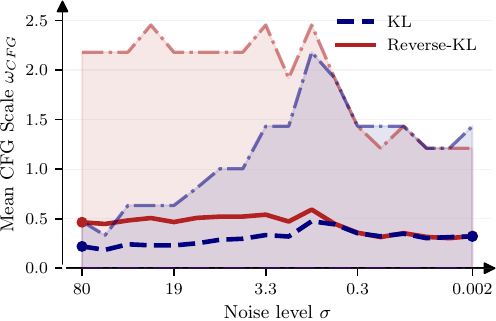}
    \caption{Learned guidance profiles $\omega(x,\sigma)$ on CIFAR-10 for KL and rKL (mean and $[0.25,0.75]$ quantiles). KL applies higher guidance to a larger fraction of samples, improving Recall; rKL concentrates guidance on fewer samples, improving Precision.}
\end{figure}

The CIFAR-10 profile mirrors the behavior observed on FFHQ (Figure~\ref{fig:cfg_profiles_ffhq}): the policy increases the proportion of highly guided samples at medium-to-low noise levels, coinciding with the speciation phase of sampling \citep{biroli_dynamical_2024}. Around medium noise levels, guidance strength exhibits higher mean and variance, indicating a greater need for state-dependent control during the phase when topological structures emerge in the images.

\FloatBarrier
\subsection{Is it necessary to have $x$-dependent policies?}
\label{app:cfg_x_dependent}

In this appendix, we analyze whether conditioning the sampling policy on the image state $x$ is necessary when learning adaptive classifier-free guidance (CFG). In our framework, this corresponds to comparing policies of the form $\pi_\theta(\omega \mid x, \sigma)$ against simpler policies that depend only on the noise level, $\pi_\theta(\omega \mid \sigma)$.

A noise-level–dependent policy learns a single guidance scale per noise level, which can be interpreted as a global trade-off between fidelity and diversity at each stage of the sampling process. While this is sufficient to recover standard hand-tuned guidance schedules, it cannot adapt guidance strength across samples. In contrast, an $x$-dependent policy can modulate the guidance scale based on the current state of the sample, allowing different images to receive different amounts of guidance at the same noise level.

Empirically, we observe that $x$-dependent policies consistently outperform $\sigma$-only policies in terms of sample quality. As reported in Table~\ref{tab:ablation_conditioning}, conditioning on $(x,\sigma)$ improves FID, Precision, and Recall across datasets and divergences. Qualitatively, the learned policies exhibit sparse behavior: at a given noise level, only a subset of samples receives strong guidance, while the majority remain weakly guided. This behavior cannot be captured by a single scalar guidance value shared across all samples.

This observation aligns with recent analyses of CFG, which suggest that excessive guidance can lead to mode collapse and loss of diversity, while insufficient guidance degrades fidelity. An $x$-dependent policy enables the sampler to selectively apply strong guidance where it is most beneficial, while preserving diversity elsewhere. Figure~\ref{fig:cfg_x_dependent} illustrates this difference at the policy-profile level: the $x$-dependent policy exhibits larger dispersion across samples, whereas the $\sigma$-only policy is forced to apply a single global schedule. In practice, this adaptivity appears crucial for achieving favorable precision--recall trade-offs without manual tuning.
\begin{figure}[H]
    \centering
    \begin{subfigure}[t]{0.49\textwidth}
        \centering
        \includegraphics[width=\linewidth]{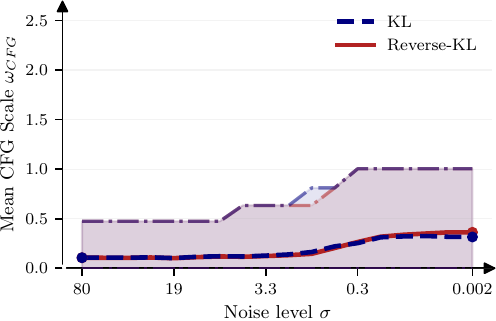}
        \caption{Noise-level-only policy $\pi(\omega\mid\sigma)$}
    \end{subfigure}
    \hfill
    \begin{subfigure}[t]{0.49\textwidth}
        \centering
        \includegraphics[width=\linewidth]{figures/mean_cfg_cifar.pdf}
        \caption{State-dependent policy $\pi(\omega\mid x,\sigma)$}
    \end{subfigure}
    \caption{Comparison of $x$-dependent and noise-level–dependent policies for adaptive CFG on CIFAR-10 using KL and Reverse-KL divergence minimization. The $x$-dependent policy achieves better FID, Precision, and Recall by selectively applying guidance based on the current sample state.
        (a) depends only on the noise level $\sigma$; (b) uses both $x$ and $\sigma$.}
    \label{fig:cfg_x_dependent}
\end{figure}

Figure~\ref{fig:viz_cfg} provides a complementary qualitative view by comparing trajectories initialized from the same noise. The fixed-guidance sampler follows a single hand-tuned schedule, while the learned state-dependent policy applies stronger guidance only at selected medium-to-low noise levels. This changes the path through the sampler state space and can steer samples toward different image content, illustrating why per-sample adaptivity can affect quality even when the same denoiser and number of solver steps are used.
\begin{figure}[H]
    \centering
    \includegraphics[width=\textwidth]{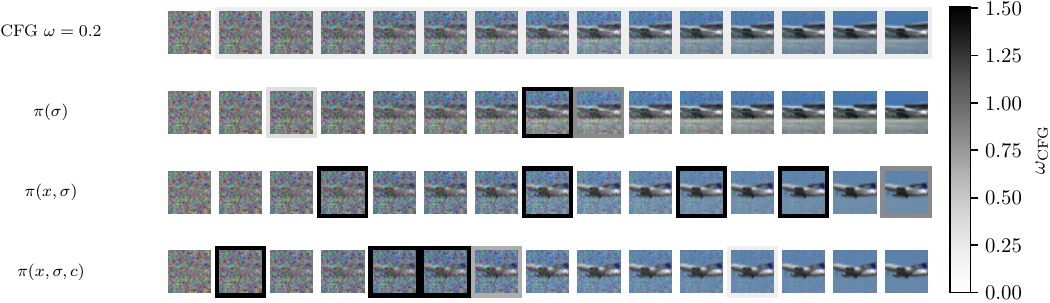}
    \caption{Example trajectory starting from the same initial noise $x$ for fixed CFG ($\omega=0.2$) and learned KL-based adaptive CFG on CIFAR-10. State-dependent policies apply high guidance at medium-to-low noise levels, steering the trajectory toward different image content.}
    \label{fig:viz_cfg}
\end{figure}

For the Gamma strategy, the comparison in Table~\ref{tab:x_dep_gammas_cifar} is more mixed: conditioning on $x$ does not consistently improve over a noise-level-only policy. This suggests that state-dependent control is especially important for guidance, while stochasticity injection can sometimes be captured by a simpler noise-level schedule.

\begin{table}[H]
    \centering
    \caption{Comparison of $x$-dependent and noise-level-dependent policies for adaptive stochasticity injection on CIFAR-10 using KL and Reverse-KL divergence minimization. Unlike adaptive CFG, $x$-dependence does not consistently improve this strategy, suggesting that the value of state-dependent control depends on the intervention being learned.}
    \label{tab:x_dep_gammas_cifar}
    {\small
\begin{tabular}{lrrrc}
\toprule
Method & FID $\downarrow$ & Prec $\uparrow$ & Rec $\uparrow$  & $x$-dependent \\
\midrule

 & 1.99 & 78.5 & \textbf{73.2} & $\times$ \\
KL& 2.03& 78.6 &	72.7 & \checkmark \\
& 2.06& 79.0 & 72.4 & \checkmark \\
\midrule
& 2.03 & \textbf{79.7} & 71.9 & $\times$ \\
rKL & 2.03&	78.6&	72.4 & \checkmark \\
&2.06 &	79.2 &	72.5 & \checkmark \\
\midrule
EDM & \textbf{1.98} & 78.7 & 72.9 & $\times$ \\
\bottomrule
\end{tabular}

}
\end{table}

\FloatBarrier
\subsection{Impact of the Choice of $f$-Divergence for CFG}
\label{app:cfg_fdivs}

Table~\ref{tab:cfg_fdivs_cifar10} reports CFG results on CIFAR-10 ($32\times32$) for additional $f$-divergences.
All choices yield FID scores on par with KL and rKL, confirming that the framework is not sensitive to this choice.
However, their qualitative behavior differs: Total Variation, Hellinger, and Jensen--Shannon are symmetric divergences and do not exhibit the same tendency as KL or rKL to concentrate guidance aggressively in specific regions of the noise space.
The $\chi^2$ divergence, sometimes favored for its smoothness in optimization, did not produce any noticeable improvement in learning curves compared to KL.

\begin{table}[H]
    \centering
    \caption{Adaptive CFG on CIFAR-10 ($32\times32$) for different $f$-divergences. KL and rKL from the main paper are shown for reference (separated by a rule).}
    \label{tab:cfg_fdivs_cifar10}
    {\small\begin{tabular}{lrrr}
    \toprule
    $f$-divergence & FID $\downarrow$ & Prec $\uparrow$ & Rec $\uparrow$ \\
    \midrule
    Total Variation  & 2.19 & 80.7 & 70.6 \\
    Hellinger        & 2.20 & 80.7 & 70.7 \\
    Jensen--Shannon  & 2.21 & 80.6 & 70.7 \\
    $\chi^2$         & 2.21 & 80.6 & 70.8 \\
    \midrule
    KL               & \textbf{2.05} & 80.7 & \textbf{71.4} \\
    rKL              & 2.53 & \textbf{82.4} & 69.8 \\
    \bottomrule
\end{tabular}
}
\end{table}

\FloatBarrier
\subsection{Learned $\gamma_{\mathrm{EDM}}$ Profiles}
\label{app:gamma_profiles}

\begin{figure}[H]
    \centering
    \includegraphics[width=0.6\textwidth]{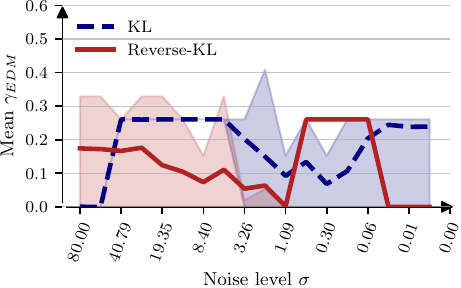}
    \caption{Learned $\gamma_{\mathrm{EDM}}(x,\sigma)$ profiles on CIFAR-10 for KL and rKL divergences (mean and $[0.25,0.75]$ quantiles). The profiles lack simple structure, confirming that optimal stochasticity injection is highly state-dependent.}
    \label{fig:gamma_profiles}
\end{figure}

On CIFAR-10, rKL adds less noise with low variance at low-noise regions to increase precision, while KL does the opposite, adding more noise with higher variance to increase diversity. This behavior is consistent with the mode-seeking/mode-covering interpretation of the divergences.

\FloatBarrier
\subsection{Temperature}
\label{app:temp_gammas}
We experimented with adjusting the temperature of the policy during evaluation to study its effect on the trade-off between sample quality and diversity.
Lowering the temperature sharpens the action distribution, leading to more deterministic behavior, while increasing it encourages exploration by flattening the policy over noise-injection actions.
As discussed in Section~\ref{sec:results}, temperature therefore provides an inference-time mechanism to control the exploration--exploitation balance of the learned policy.

Figure~\ref{fig:temp_gammas_cifar} shows that increasing temperature leads to higher Recall but also higher FID.
For the Gamma strategy, this behavior can be attributed to more frequent stochasticity injection at higher noise levels, which increases sample diversity but also perturbs trajectories that would otherwise converge to high-fidelity modes.
As a result, diversity gains come at the cost of reduced sample sharpness.

\begin{figure}[H]
    \centering
    \includegraphics[width=\textwidth]{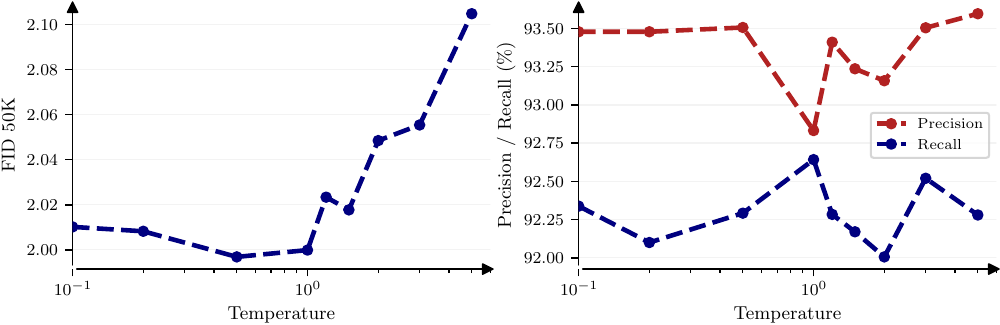}
    \caption{Effect of temperature scaling on the learned Gamma sampling policy for CIFAR-10 using KL divergence minimization.}
    \label{fig:temp_gammas_cifar}
\end{figure}

Notably, we observe that Precision and Recall are jointly optimized near the temperature used during training ($\beta = 1$).
This suggests that the learned policy is already calibrated to balance exploration and exploitation under the training objective, and that deviating significantly from this temperature primarily trades off fidelity for diversity rather than uncovering strictly better operating points.
\FloatBarrier
\subsection{Controlling the Distribution of Generated Samples}
We also considered settings where the output distribution of a pretrained model is biased and must be corrected at inference time.
We focused on renoising strategies. On CIFAR-10, although the dataset is balanced, pretrained EDM models overproduce animal classes relative to vehicle classes. We defined a target class distribution
that increases the mass of vehicle classes while preserving normalization. Figure~\ref{fig:class_control_restart} shows that the learned renoise policy shifts the generated class distribution toward the target. While the match is not exact,
the bias is substantially reduced, indicating that the policy captures the intended correction direction.
\begin{figure}[H]
    \centering
    \includegraphics[width=0.55\textwidth]{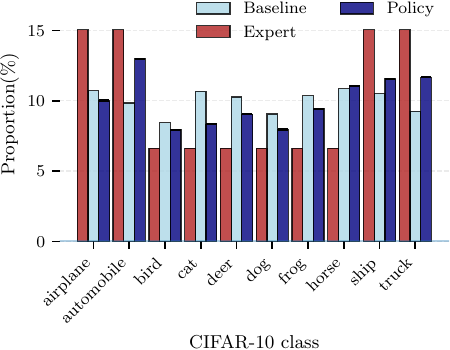}
    \caption{Class distribution of samples generated by a pretrained EDM model on CIFAR-10, before and after applying a learned renoise policy to correct class imbalance. The target distribution increases the mass of vehicle classes (automobile, truck) relative to animal classes. The learned policy shifts the generated distribution closer to the target, reducing bias while maintaining sample quality.}
    \label{fig:class_control_restart}
\end{figure}

A key modeling choice is whether the control policy depends explicitly on the trajectory step of the induced MDP. We distinguish between non-stationary
policies $\pi_\theta(x, \sigma, t)$, which depend on the step index, and stationary policies $\pi_\theta(x, \sigma)$, which do not. Although the optimal policy for a finite sampling horizon is generally non-stationary, we find that stationary policies are significantly easier to train and often
perform better in practice. Across all strategies, stationary policies consistently match or outperform non-stationary ones, yielding more stable optimization and
better trade-offs between distributional control and sample quality at comparable NFE.

\FloatBarrier
\subsection{Dimension of the Action Space}
\label{app:action_space}

The action space $\mathcal{A}$ consists of a discrete set of candidate values for the controlled hyperparameter (e.g., guidance scales or noise injection levels). Table~\ref{tab:ablation_n_gamma} reports the effect of the cardinality $|\mathcal{A}|$ on the stochasticity injection experiment on CIFAR-10.

A small action space ($|\mathcal{A}|=5$) is insufficient to capture the range of useful noise injection levels and yields noticeably worse FID. Performance saturates quickly: $|\mathcal{A}|=20$ already matches $|\mathcal{A}|=100$ across all metrics, indicating that a moderate number of actions is sufficient for the policy to express the necessary behavior.
\begin{table}[H]
    \centering
    \caption{Effect of action space cardinality $|\mathcal{A}|$ on the stochasticity injection strategy on CIFAR-10.}
    {\small
\begin{tabular}{lrrr}
\toprule
$\text{card}(A)$ & FID $\downarrow$ & Prec $\uparrow$ & Rec $\uparrow$\\
\midrule
5 & 2.20 &	79.8 &	72.4 \\
20 & 2.06 &	78.8 &	72.5 \\
100  & 2.07	&78.7 &	72.8 \\
\bottomrule
\end{tabular}

}
    \label{tab:ablation_n_gamma}
\end{table}

\FloatBarrier
\subsection{Role of the ADM Encoder in the Backbone}
\label{app:adm_backbone}

Our policy network uses a pretrained ADM encoder as a frozen feature extractor to condition the policy on the current sample $x$. Table~\ref{tab:adm_or_not} ablates this design choice on the CFG experiment on CIFAR-10.

Removing the ADM encoder leads to a significant drop in performance: FID degrades from 2.05 to 3.04, Recall drops by nearly 4 points, while Precision increases, suggesting the policy collapses toward high-fidelity but less diverse modes. This confirms that rich, pretrained visual features are essential for the policy to accurately distinguish sample quality across different states and noise levels. Using the ADM encoder as a frozen backbone is therefore a key component of the architecture.

\begin{table}[H]
    \centering
    \caption{Impact of the ADM encoder on the CFG experiment on CIFAR-10. \cmark: ADM encoder used; \xmark: policy conditioned on raw $x$ only.}
    \label{tab:adm_or_not}
    {\small
\begin{tabular}{lrrr}
\toprule
ADM & FID $\downarrow$ & Prec $\uparrow$ & Rec $\uparrow$\\
\midrule
\xmark & 3.04 &	82.6 &	67.5	\\
\cmark & 2.05 & 80.7 & 71.4 \\
\bottomrule
\end{tabular}

}
\end{table}

\FloatBarrier
\section{Broader Impact}
\label{app:broader_impact}

This work proposes a framework for learning sampling policies for diffusion models via inverse reinforcement learning, with the goal of improving image quality and controllability without retraining the underlying generative model. Our method introduces a learned policy and a discriminator trained jointly with it, which adds a modest computational overhead at training time. We believe this cost is justified by the elimination of manual hyperparameter search, and we have documented it empirically in Table~\ref{tab:compute_cost}. We do not foresee direct societal harms specific to this work beyond those already associated with diffusion-based image generation more broadly.


\newpage
\section*{NeurIPS Paper Checklist}

\begin{enumerate}

\item {\bf Claims}
    \item[] Question: Do the main claims made in the abstract and introduction accurately reflect the paper's contributions and scope?
    \item[] Answer: \answerYes{}
    \item[] Justification: The abstract and introduction clearly state the three contributions: (1) an IRL framework for learning diffusion sampling policies via state occupancy matching, (2) a policy gradient theorem enabling practical optimization, and (3) empirical validation on CIFAR-10, FFHQ, and ImageNet across three sampling strategies. All claims are supported by experiments in Section~\ref{sec:results}.

\item {\bf Limitations}
    \item[] Question: Does the paper discuss the limitations of the work performed by the authors?
    \item[] Answer: \answerYes{}
    \item[] Justification: Limitations are discussed in the Conclusion: sensitivity to initialization, dependence on the design of the discrete action space, and the restriction to discrete actions (continuous action spaces are left for future work).

\item {\bf Theory assumptions and proofs}
    \item[] Question: For each theoretical result, does the paper provide the full set of assumptions and a complete (and correct) proof?
    \item[] Answer: \answerYes{}
    \item[] Justification: Theorem~\ref{thm:grad_inf} (policy gradient for $f$-divergence minimization) is stated with explicit assumptions and proved in full in Appendix~\ref{proof:thm_grad}. All intermediate results are numbered and cross-referenced.

    \item {\bf Experimental result reproducibility}
    \item[] Question: Does the paper fully disclose all the information needed to reproduce the main experimental results of the paper to the extent that it affects the main claims and/or conclusions of the paper (regardless of whether the code and data are provided or not)?
    \item[] Answer: \answerYes{}
    \item[] Justification: Full training details (architecture, optimizer, hyperparameters, action spaces, training schedules) are provided in Appendix~\ref{app:experimental_details}. All experiments use publicly available pretrained EDM models and standard datasets.

\item {\bf Open access to data and code}
    \item[] Question: Does the paper provide open access to the data and code, with sufficient instructions to faithfully reproduce the main experimental results, as described in supplemental material?
    \item[] Answer: \answerYes{}
    \item[] Justification: Code is provided as an anonymized zip file in the supplemental material. All datasets (CIFAR-10, FFHQ, ImageNet) and pretrained models (EDM) are publicly available.

\item {\bf Experimental setting/details}
    \item[] Question: Does the paper specify all the training and test details (e.g., data splits, hyperparameters, how they were chosen, type of optimizer) necessary to understand the results?
    \item[] Answer: \answerYes{}
    \item[] Justification: Full experimental details are provided in Appendix~\ref{app:experimental_details}, including policy architectures, optimizers, learning rates, action spaces, discriminator training, and evaluation protocols for each strategy and dataset.

\item {\bf Experiment statistical significance}
    \item[] Question: Does the paper report error bars suitably and correctly defined or other appropriate information about the statistical significance of the experiments?
    \item[] Answer: \answerYes{}
    \item[] Justification: The main table reports point estimates for readability, while Appendix~\ref{app:results_std} reports mean$\pm$std values when repeated CIFAR-10 and FFHQ evaluations are available (Tables~\ref{tab:all_results_std} and~\ref{tab:ablation_conditioning_std}). ImageNet results are single-run evaluations due to the high computational cost of 50k-sample metrics at that scale.

\item {\bf Experiments compute resources}
    \item[] Question: For each experiment, does the paper provide sufficient information on the computer resources (type of compute workers, memory, time of execution) needed to reproduce the experiments?
    \item[] Answer: \answerYes{}
    \item[] Justification: Table~\ref{tab:compute_cost} reports training and inference costs in EFLOPs for all strategies and datasets. GPU type, memory, and wall-clock times are reported in Appendix~\ref{app:experimental_details}.

\item {\bf Code of ethics}
    \item[] Question: Does the research conducted in the paper conform, in every respect, with the NeurIPS Code of Ethics \url{https://neurips.cc/public/EthicsGuidelines}?
    \item[] Answer: \answerYes{}
    \item[] Justification: The paper uses publicly available datasets and pretrained models. No human subjects are involved. The work does not introduce new risks beyond those already present in existing generative models.

\item {\bf Broader impacts}
    \item[] Question: Does the paper discuss both potential positive societal impacts and negative societal impacts of the work performed?
    \item[] Answer: \answerNA{}
    \item[] Justification: This is foundational research on improving sampling efficiency and quality of diffusion models. It does not introduce new application domains or deployment-ready systems. Improving generative model quality could in principle facilitate misuse (e.g., deepfakes), but this concern applies broadly to all work on generative models and is not specific to this contribution.

\item {\bf Safeguards}
    \item[] Question: Does the paper describe safeguards that have been put in place for responsible release of data or models that have a high risk for misuse (e.g., pre-trained language models, image generators, or scraped datasets)?
    \item[] Answer: \answerNA{}
    \item[] Justification: The paper does not release new pretrained models or datasets. It builds on existing publicly available EDM checkpoints and standard benchmarks (CIFAR-10, FFHQ, ImageNet).

\item {\bf Licenses for existing assets}
    \item[] Question: Are the creators or original owners of assets (e.g., code, data, models), used in the paper, properly credited and are the license and terms of use explicitly mentioned and properly respected?
    \item[] Answer: \answerYes{}
    \item[] Justification: The EDM pretrained models \citep{karras_elucidating_2022} are used under their MIT license. CIFAR-10, FFHQ, and ImageNet are standard benchmarks cited appropriately. The ADM encoder \citep{dhariwal_diffusion_2021} is used under its MIT license.

\item {\bf New assets}
    \item[] Question: Are new assets introduced in the paper well documented and is the documentation provided alongside the assets?
    \item[] Answer: \answerNA{}
    \item[] Justification: The paper does not introduce new datasets or model checkpoints. The method (code) will be released upon acceptance.

\item {\bf Crowdsourcing and research with human subjects}
    \item[] Question: For crowdsourcing experiments and research with human subjects, does the paper include the full text of instructions given to participants and screenshots, if applicable, as well as details about compensation (if any)?
    \item[] Answer: \answerNA{}
    \item[] Justification: The paper does not involve crowdsourcing or human subjects.

\item {\bf Institutional review board (IRB) approvals or equivalent for research with human subjects}
    \item[] Question: Does the paper describe potential risks incurred by study participants, whether such risks were disclosed to the subjects, and whether Institutional Review Board (IRB) approvals (or an equivalent approval/review based on the requirements of your country or institution) were obtained?
    \item[] Answer: \answerNA{}
    \item[] Justification: The paper does not involve human subjects; no IRB approval is required.

\item {\bf Declaration of LLM usage}
    \item[] Question: Does the paper describe the usage of LLMs if it is an important, original, or non-standard component of the core methods in this research? Note that if the LLM is used only for writing, editing, or formatting purposes and does \emph{not} impact the core methodology, scientific rigor, or originality of the research, declaration is not required.
    \item[] Answer: \answerNA{}
    \item[] Justification: LLMs were used solely for writing and reformulation assistance. They were not involved in the core methodology, experimental design, or scientific analysis, and therefore declaration is not required per NeurIPS guidelines.

\end{enumerate}

\end{document}